\documentclass{article}

\usepackage{microtype}
\usepackage{graphicx}
\usepackage{subcaption}
\usepackage{booktabs} 

\usepackage{hyperref}

\usepackage[accepted]{icml2024}

\usepackage{amsmath}
\usepackage{amssymb}
\usepackage{mathtools}
\usepackage{amsthm}

\usepackage[capitalize,noabbrev]{cleveref}

\theoremstyle{plain}
\newtheorem{theorem}{Theorem}[section]

\theoremstyle{definition}
\newtheorem{definition}[theorem]{Definition}

\theoremstyle{remark}

\usepackage[textsize=tiny]{todonotes}

\renewcommand{\vec}[1]{\boldsymbol{#1}}

\usepackage{enumitem,nicefrac}

\usepackage{tablefootnote}
\usepackage{calc} 
\usepackage{pifont} 
\usepackage{xcolor}
\usepackage{colortbl} 
\definecolor{cmarkcolor}{RGB}{0, 128, 0}
\definecolor{acmarkcolor}{RGB}{52, 152, 219}
\definecolor{xmarkcolor}{RGB}{200, 0, 0}
\newcommand{\cmark}{{\color{cmarkcolor} \ding{51}}}   
\newcommand{\xmark}{{\color{xmarkcolor}\ding{55}}}      

\icmltitlerunning{Is Epistemic Uncertainty Faithfully Represented by Evidential Deep Learning Methods?}

\begin{document}

\twocolumn[
\icmltitle{Is Epistemic Uncertainty Faithfully Represented \\ by Evidential Deep Learning Methods?}



\icmlsetsymbol{equal}{*}

\begin{icmlauthorlist}
\icmlauthor{Mira Jürgens}{yyy}
\icmlauthor{Nis Meinert}{comp}
\icmlauthor{Viktor Bengs}{sch}
\icmlauthor{Eyke Hüllermeier}{sch}
\icmlauthor{Willem Waegeman}{yyy}
\end{icmlauthorlist}

\icmlaffiliation{yyy}{Department of Data Analysis and Mathematical Modeling, Ghent University, Belgium}
\icmlaffiliation{comp}{Institute of Communications and Navigation, German Aerospace Center (DLR), Neustrelitz, Germany}
\icmlaffiliation{sch}{Department of Informatics, University of Munich (LMU), Germany}

\icmlcorrespondingauthor{Mira Jürgens}{mira.juergens@ugent.be}

\icmlkeywords{Machine Learning, ICML}

\vskip 0.3in
]



\printAffiliationsAndNotice{} 

\begin{abstract}
    Trustworthy ML systems should not only return accurate predictions, but also a reliable representation of their uncertainty. Bayesian methods are commonly used to quantify both aleatoric and epistemic uncertainty, but alternative approaches, such as evidential deep learning methods, have become popular in recent years. 
    The latter group of methods in essence extends empirical risk minimization (ERM) for predicting second-order probability distributions over outcomes, from which measures of epistemic (and aleatoric) uncertainty can be extracted.
    This paper presents novel theoretical insights of evidential deep learning, highlighting the difficulties in optimizing second-order loss functions and interpreting the resulting epistemic uncertainty measures.
    With a systematic setup that covers a wide range of approaches for classification, regression and counts, it provides novel insights into issues of identifiability and convergence in second-order loss minimization, and the relative (rather than absolute) nature of epistemic uncertainty measures.
\end{abstract}

\section{Introduction}
Machine learning systems are capable of making predictions based on patterns found in data.
However, these patterns may be incomplete, noisy, or biased, which can lead to errors and inaccuracies in the system's predictions.
To address this challenge, uncertainty quantification has emerged as a technique to measure the reliability of a model's predictions and the confidence that can be placed in its results.
In supervised learning, uncertainty is often classified into two categories: aleatoric and epistemic~\citep{kend_wu17,mpub440}.
Since outcomes cannot be predicted deterministically, we assume for classical supervised learning methods a conditional probability $p(y \,|\, \vec{x})$ on $\mathcal{Y}$ as the ground truth and train a probabilistic predictor to produce estimates $\hat{p}(y \,|\, \vec{x}) $.
These estimates capture the aleatoric uncertainty about the actual outcome $y \in \mathcal{Y}$, which refers to the inherent randomness that the learner cannot eliminate.
However, this approach does not allow the learner to express its epistemic uncertainty, which is the lack of knowledge about how accurately $\hat{p}$ approximates $p$. 

 To capture both types of uncertainty, a popular choice is to consider second-order distributions, which map a query instance $\vec{x}$ to a probability distribution over all possible candidate distributions $p$. This concept forms one of the core concepts behind Bayesian methods, such as Bayesian neural networks, together with the incorporation of domain knowledge via a prior distribution~\citep{depe_du18}. However, exact computation of the posterior distribution can be computationally intractable, particularly when using deep learning methods.
As a result, approximate techniques such as variational inference and sampling-based methods have been used extensively~\citep{Graves2011}.

\begin{table*}[t]

\caption{Overview of most important methods analyzed in this paper. See text for a further explanation of the column headers (NLL = negative log-likelihood, MAE = mean absolute error, OOD = out-of-distribution, KL = Kullbeck-Leibler divergence).}
\label{tab:methods}
\vskip 0.15in
\begin{scriptsize}    
\begin{sc}
\begin{tabular}{lccccc}
\toprule
Method & Optimization & First-order & Setting & Regularizer & Quantification of \\
&Problem& Loss $L_1$&& $R(\vec{\Phi})$& Epistemic Uncertainty \\
\midrule
\citet{sens_ed18} & Outer & Brier & Classification & Entropy & Pseudo-counts\\
\citet{MalininG18} & Outer & KL & Classification & OOD KL & Mutual information\\
\citet{MalininG19} & Outer  & Reverse KL & Classification & OOD KL & Mutual information \\
\citet{char_ue20} & Outer & NLL & Classification & Entropy & Pseudo-counts \\
\citet{HuseljicSHK20} & Inner & NLL & Classification & OOD KL & Pseudo-counts \\
\citet{MalininUnpub} & Outer & NLL & Regression & Entropy & Mutual information\\
\citet{amini2020deep} & Inner &NLL& Regression & Specific & Variance \\
\citet{Ma2021TrustworthyMR} &Inner & NLL & Regression & Sepcific & Variance \\
\citet{tsiligkaridis2021information} & Outer & $L_p$-norm & Classification & R\'enyi divergence & Mutual information \\
\citet{KopetzkiCZGG21} & Outer & Multiple & Classification & Entropy & Pseudo-counts \\
\citet{Bao2021} & Outer & Brier & Classification & Specific & Pseudo-counts \\
\citet{Liu2021} & Inner & NLL+MAE & Regression & Pseudo-counts & Variance \\
\citet{Charpentier2022NaturalPN} & Outer & NLL & Class./Reg./Counts & Entropy & Pseudo-counts \\
\citet{Oh_Shin_2022} & Inner & Squared loss & Regression & Gradient-based & Variance \\
\citet{Shankar2023AAAI} & Inner & NLL & Regression & Distance to OOD  & Variance\\
\citet{Kotelevskii2023} & Outer & Specific & Classification & Entropy & Entropy \\
\bottomrule
\end{tabular}
\end{sc}
\end{scriptsize}
\vskip -0.1in
\end{table*}

\begin{table*}[ht!]
	\caption{%
		Overview of the key points covered in previous work that critically scrutinizes evidential deep learning (EDL). Detailed differences are discussed in Section~\ref{sec:theory}. 
	}\label{tbl:hil-survey-overview}
\label{tab:overview_contribution}
\vskip 0.15in
\begin{scriptsize}     
\begin{sc}
\begin{tabular}{lccccc}
		\toprule
		\textbf{Reference} & \textbf{Critique} & \textbf{Classification} & \textbf{Regression} & \textbf{Count} & \textbf{Experiments} \\
		\midrule
		\citet{bengs2022pitfalls} & Convergence issues of EDL & \cmark & \xmark & \xmark & \cmark \\
		\citet{Meinert2023} &  Convergence issues of EDL & \xmark & \cmark & \xmark & \cmark \\[0.3ex]
		\citet{Bengs2023} & Lack of properness of second-order loss functions & \cmark & \cmark & \xmark & \xmark \\[0.3ex]
            \textbf{Our work} & No coherence with a suitable reference distribution  & 
        \cmark & \cmark & \cmark & \cmark \\
		\bottomrule
\end{tabular}
\end{sc}
\end{scriptsize}
\vskip -0.1in
\end{table*}

An alternative approach, known as evidential deep learning, that adopts a direct uncertainty quantification by means of second-order risk minimization, is currently gaining a lot of traction (see Table~\ref{tab:methods} for an overview of the most important methods). 
This approach takes a frequentist statistics perspective and introduces specific loss functions to estimate a second-order distribution over $\hat{p}(y \,|\, \vec{x}) $, without the need for specifying a prior, unlike Bayesian methods. In the last years, evidential deep learning has been widely used in applications of supervised learning, and many authors claim that it is state-of-the-art in downstream ML tasks that benefit from epistemic uncertainty quantification, such as out-of-distribution detection~\citep{Yang2022}, robustness to adversarial attacks~\citep{KopetzkiCZGG21} or active learning~\citep{Park2023}.
However, in all these tasks, one only needs \emph{relative} epistemic uncertainty measures, e.g., when an example attains a higher epistemic uncertainty than a second example, it is more likely to be out-of-distribution. 
For specific optimization problems, some authors have expressed concerns in different ways as to whether such measures also have a quantitative interpretation. Nonetheless, previous papers have not covered all possible aspects, see Table \ref{tab:overview_contribution} for a general overview. 

This paper presents novel theoretical results that put previous papers in a unifying perspective by analyzing a broad class of exponential family models that cover classification, regression and count data. More specifically, we intend to answer the following two questions: (a) What properties are needed for second-order risk minimization methods to represent epistemic uncertainty in a quantitatively faithful manner? and (b) Do such methods represent epistemic uncertainty in a faithful manner? For the first research question, we will introduce in Section~\ref{subsec:reference_distribution} the notion of a \emph{reference distribution}, and we will argue that second-order methods should provide an estimate of this distribution if they intend to model epistemic uncertainty in a frequentist manner. For the second research question, we will present in Sections~\ref{subsec:unreg_inner_min}-\ref{subsec:reg_risk_min} and Section~\ref{sec:experiments} novel theoretical and experimental results, which show that the answer to this question is negative, i.e., evidential deep learning methods do not represent epistemic uncertainty in a faithful manner.  




\section{Problem definition and literature review}
Consider a classification or regression task with instance space $\mathcal{X}$ and label space $\mathcal{Y}$.
Without loss of generality, we assume $\mathcal{X} = \mathbb{R}^D$ with $D$ the dimensionality of the data.
Similarly, we assume $\mathcal{Y} = \mathbb{R}^{D'}$ in the case of regression, $\mathcal{Y} = \{1,\ldots,K\}$ in $K$-class classification, and $\mathcal{Y} = \mathbb{N}$ in the case of counts. 
Furthermore, we have training data of the form $\mathcal{D} = \big\{ \big(\vec{x}_i , y_i \big) \big\}_{i=1}^N$, in which each $(\vec{x}_i , y_i)$ is i.i.d.\ sampled via a joint probability measure $P$ on $\mathcal{X} \times \mathcal{Y}$.
Correspondingly, each feature vector $\vec{x} \in \mathcal{X}$ is associated with a conditional distribution $p(\cdot \,|\, \vec{x})$ on $\mathcal{Y}$, such that $p(y \,|\, \vec{x})$ is the probability to observe label $y$ as an outcome given $\vec{x}$. 
More generally, $\mathbb{P}(A)$ will denote the set of all probability distributions on the set $A$. 

\subsection{First-order risk minimization}
The classical statistics and machine learning literature usually considers first-order distributions during risk minimization. More specifically, one has a hypothesis space $\mathcal{H}_1$ containing $\Theta \rightarrow \mathbb{P}(\mathcal{Y})$ mappings with $\Theta$ a parameter space. In fact, $\mathcal{H}_1$ defines probability density (or mass) functions $p(y \,|\, \vec{\theta})$ from a parametric family with parameter vector $\vec{\theta} \in \Theta$, for continuous (or discrete) random variables. In a supervised learning context, $\vec{\theta}$ is a function of the features: we will use the notation $\vec{\theta}(\vec{x}; \vec{\Phi})$, being an $\mathcal{X} \rightarrow \Theta$ function that is parameterized by a set of parameters $\vec{\Phi}$. Risk minimization is then usually defined as
\begin{equation}
    \label{eq:exprisktradi}
		\min_{\vec{\Phi}} \sum_{i=1}^N L_1 \!\left( y_i,  p(y \,|\, \vec{\theta}(\vec{x}_i; \vec{\Phi})) \right).
\end{equation}
 As a loss function $L_1: \mathcal{Y} \times \mathbb{P}(\mathcal{Y}) \rightarrow \mathbb{R}$, often the negative log-likelihood is chosen:
\begin{equation*}
    L_1(y, p(y \,|\, \vec{\theta}(\vec{x}; \vec{\Phi}))) = -\log(p(y \,|\, \vec{\theta}(\vec{x}; \vec{\Phi})) \,.
\end{equation*}
For classical machine learning models, usually a regularization term that penalizes $\vec{\Phi}$ is incorporated in (\ref{eq:exprisktradi}) as well, but this is less common for modern deep learning models. 

\subsection{Second-order risk minimization}
In second-order risk minimization, one defines a second hypothesis space $\mathcal{H}_2$ with $\mathcal{M} \rightarrow \mathbb{P}(\Theta)$ mappings with parameter space $\mathcal{M}$. Thus, $\mathcal{H}_2$ contains probability density functions  $p(\vec{\theta} \,|\, \vec{m})$ from a parametric family with parameter vector $\vec{m}(\vec{x}; \vec{\Phi}) \in \mathcal{M}$ that is a function of $\vec{x}$. The function $\vec{m}(\vec{x}; \vec{\Phi})$, in turn, will also be parameterized by a set of parameters $\vec{\Phi}$, which will be estimated using a training dataset. Below we will show concretely how $\vec{m}(\vec{x}; \vec{\Phi}) \in \mathcal{M}$ can be parameterized in the case of neural networks. 

As shown in Table~\ref{tab:methods}, one considers two different types of optimization problems in second-order risk minimization.
In a first branch of papers, which will be further referred to as \emph{inner loss minimization}, one computes an expectation over $p(\vec{\theta} \,|\, \vec{m}(\vec{x}_i; \vec{\Phi}))$ inside the loss function.
Let $L_1: \mathcal{Y} \times \mathbb{P}(\mathcal{Y})$ be a traditional loss function, and let $R(\vec{\Phi})$ be a penalty term with regularization parameter $\lambda \in \mathbb{R}_{\geq 0}$, then the empirical risk is minimized as follows: 
\begin{equation}
    \label{eq:exprisk1reg}
    \min_{\vec{\Phi}} \sum_{i=1}^N L_1 \!\left( y_i, \mathbb{E}_{\vec{\theta} \sim p(\vec{\theta} \,|\, \vec{m}(\vec{x}_i; \vec{\Phi}))} \!\left[ p(y \,|\, \vec{\theta})\right] \right) + \lambda R(\vec{\Phi}).
\end{equation}
Most of the papers with this formulation focus on the regression setting, adopting various mechanisms for the penalty term $R(\vec{\Phi})$. 

In a second branch of papers, which will be further referred to as \emph{outer loss minimization}, one computes an expectation over $p(\vec{\theta} \,|\, \vec{m}(\vec{x}_i; \vec{\Phi}))$ outside the loss function.
The risk is then defined and minimized in a very similar way: 
\begin{equation}
    \label{eq:exprisk2reg}
    \min_{\vec{\Phi}} \sum_{i=1}^N \mathbb{E}_{\vec{\theta} \sim p(\vec{\theta} \,|\, \vec{m}(\vec{x}_i;\vec{\Phi}))} \!\left[ L_1(y_i, p(y \,|\, \vec{\theta})) \right] + \lambda R(\vec{\Phi}) \,.
\end{equation}
Papers with this formulation mainly focus on the classification setting. Here, too, one constructs various mechanisms for the penalty term $R(\vec{\Phi})$. 

Second-order risk minimization is often presented as an alternative for Bayesian inference in neural networks. Given a training dataset $\mathcal{D}$ and a prior $p(\vec{\Phi})$, Bayesian methods maintain a posterior distribution $p(\vec{\Phi} \,|\, \mathcal{D})$ over the network parameters $\vec{\Phi}$. From the posterior one can also obtain a distribution over the parameter vector $\vec{\theta}$:
$$p(\vec{\theta} \,|\, \vec{x}) = \int \! \vec{\theta}(\vec{x};\vec{\Phi}) \, p(\vec{\Phi} \,|\, \mathcal{D}) \, d\vec{\Phi} \quad \mbox{for all $\vec{x} \in \mathcal{X}.$ }$$

For Bayesian neural networks, computing the posterior is usually very challenging, but deep evidential learning methods avoid the usage of sophisticated inference algorithms \citep{Ulmer2023}: Instead of defining a prior over the weights of the neural network, one can also define a prior over the parameters $\vec{\theta}$ of the first-order probability distribution \citep{biss_ag16}. That is why in evidential deep learning the conjugate prior of $p(y \,|\, \vec{\theta})$ is often picked as a hypothesis space for $p(\vec{\theta} \,|\, \vec{m})$.

\subsection{The exponential family} \label{subsec:exponential_family}
All previous studies consider exponential family distributions as the underlying hypothesis space for $\mathcal{H}_1$ and $\mathcal{H}_2$. The exponential family can be used to represent classification, regression and count data using Categorical, Normal, and Poisson distributions, respectively. 
For clarity we briefly review the most popular parameterizations.  

In multi-class classification, $\mathcal{H}_1$ contains all categorical distributions with parameter vector $\vec{\theta} = (\theta_1, \ldots , \theta_K) \in \triangle^K$ for $\triangle^K$ being the probability $(K-1)$-simplex. As a conjugate prior, $\mathcal{H}_2$ contains all Dirichlet distributions with parameter vector $\vec{m} = (m_1,\ldots,m_K) \in \mathbb{R}_+^K $.  
With neural networks it is natural to consider  $m_k(\vec{x}; \vec \Phi) = \exp\{\vec{w}_k^\top \vec{\Psi}(\vec{x})\}$ with $\vec{w}_k$ a parameter vector per class, and $\vec{\Psi} : \mathbb{R}^D \rightarrow \mathbb{R}^C$ a function that maps the input to an embedding via multiple feature learning layers.
The parameter set $\vec{\Phi}$ is then given by $\{\vec{w}_1,\ldots,\vec{w}_K\}$ and all parameters learned in the embedding $\vec{\Psi}$. 
Sometimes the architecture is slightly changed, so that $m_1,...,m_K$ are lower bounded by one \citep{KopetzkiCZGG21}. 

In regression, $\mathcal{H}_1$ typically represents the set of Gaussian distributions with parameters $\vec{\theta} = (\mu,\sigma^2)$. $\mathcal{H}_2$ contains all normal-inverse Gamma (NIG) distributions with parameters $\vec{m} = (\gamma,\nu,\alpha,\beta)$.  
Second-order risk minimization is typically realized via a neural network architecture with one or multiple feature learning layers, represented by a mapping $\vec{\Psi}: \mathbb{R}^D \rightarrow \mathbb{R}^C$, followed by four output neurons with functions $\gamma(\vec{x}; \vec{\Psi}) = \vec{w}_1^\top \vec{\Psi}(\vec{x})$, $\nu(\vec{x} ; \vec{\Psi}) = \exp \!\left\{ \vec{w}_2^\top \vec{\Psi}(\vec{x}) \!\right\}$, $\alpha(\vec{x} ; \vec{\Psi}) = \exp \!\left\{ \vec{w}_3^\top \vec{\Psi}(\vec{x}) \!\right\}$ and,
     $\beta(\vec{x} ; \vec{\Psi}) = \exp \!\left\{ \vec{w}_4^\top \vec{\Psi}(\vec{x}) \!\right\}$. 
Here the exponential function guarantees that the parameters $\nu$, $\alpha$ and $\beta$ are positive. Sometimes the softplus function is used as an alternative~\citep{amini2020deep}.
The parameter set $\vec{\Phi}$ is hence given by $\{\vec{w}_1,\vec{w}_2,\vec{w}_3,\vec{w}_4\}$ and all the parameters of the feature learning layers.

For binary classification problems, $\mathcal{H}_1$ and $\mathcal{H}_2$ simplifiy to Bernoulli and Beta distributions. 
For multivariate regression problems, $\mathcal{H}_1$ and $\mathcal{H}_2$ generalize  to multivariate Gaussian and normal-inverse Wishart distributions, respectively \citep{Meinert2022,Bramlage2023}. 
For count data, \citet{Charpentier2022NaturalPN} assume that $\mathcal{H}_1$ contains all Poisson distributions with rate parameter $\theta$, while $\mathcal{H}_2$ contains all Gamma distributions with parameters $\vec{m} = (\alpha,\beta)$. In all these cases, the parameter vector $\vec{m}$ can be learned as the output neurons of a neural network. Figure~\ref{fig:architecture} shows an example of a potential neural network architecture for the binary classification case. 

\begin{figure}
    \centering
    \includegraphics[width=.45\textwidth]{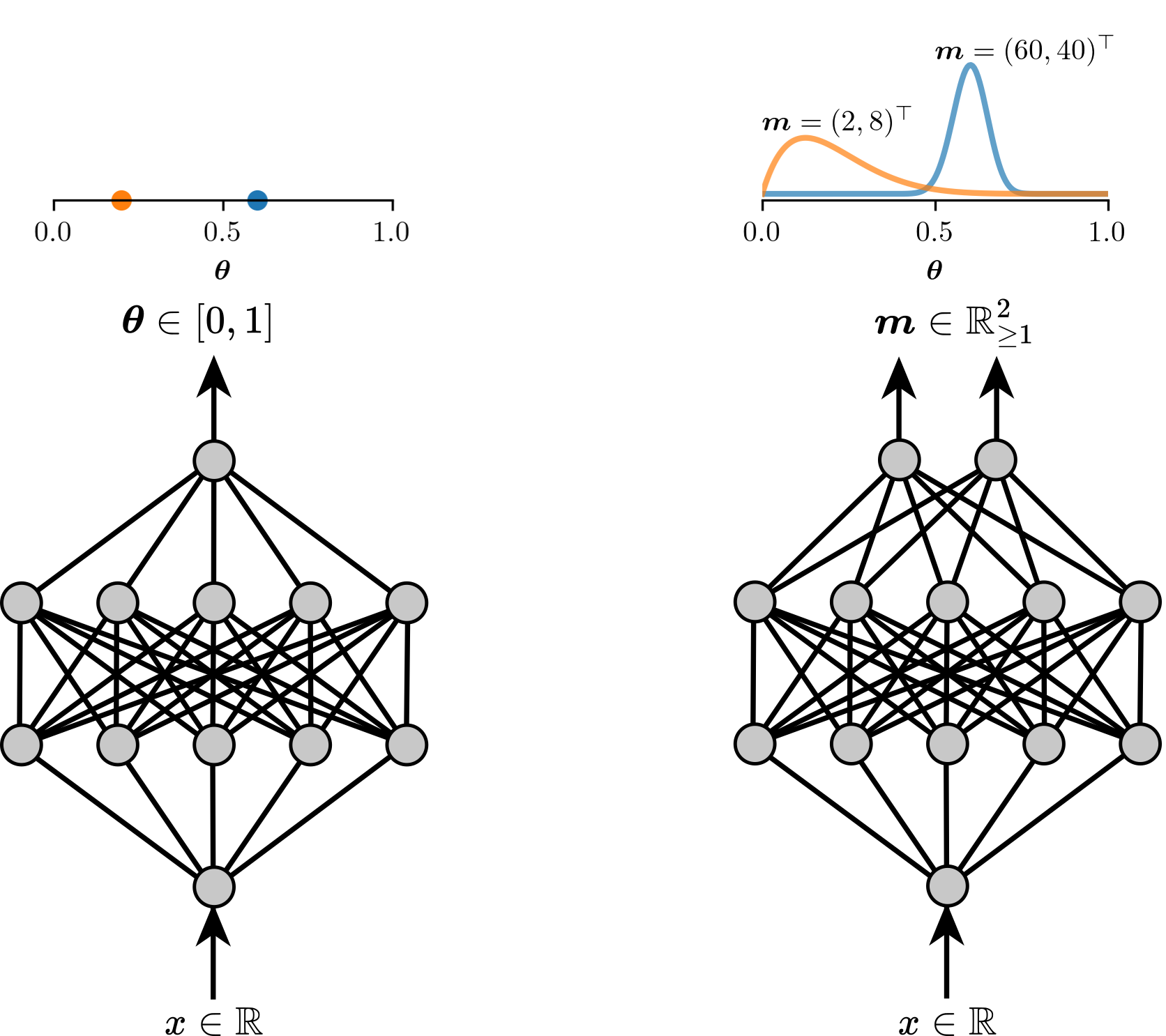}
    \caption{General overview of a neural network architecture for first-order (left) and second-order (right) risk minimization. While the first-order model learns the Bernoulli parameter $\theta$ of the data generating distribution directly, the second-order model predicts the parameters $\vec{m}=(\alpha, \beta)^\top$ of a Beta distribution, which defines a distribution with density $f(\theta|x, \alpha, \beta) = \frac{1}{B(\alpha, \beta)}\theta^{\alpha -1}(1 - \theta)^{\beta -1}$ over $\theta$. The example is aligned with our experimental setup, where the feature space is one-dimensional. }
    \label{fig:architecture}
\end{figure}

\subsection{Quantifying epistemic uncertainty}

In the literature, there is no consensus on how to quantify epistemic uncertainty with second-order risk minimization methods, because various metrics are used. As shown in Table~\ref{tab:methods}, the number of pseudo-counts, defined by $\sum_{k=1}^K m_k$, is a popular metric for classification problems. In regression papers it is defined differently, referred to as \emph{evidence}, and only used as a regularizer~\citep{amini2020deep}. In the latter case, the number of pseudo-counts is given by $\nu$ and $2 \alpha$ \citep{Gelman2013,Meinert2023}, but only $\nu$ is used for estimating the epistemic uncertainty.
This feels unnatural, since both parameters are necessary to capture the spread of the second-order distribution $p(\vec{\theta} \,|\, \vec{m}(\vec{x}; \vec{\Phi}))$.

Mutual information is also often used to quantify epistemic uncertainty of second-order loss minimization methods---as well as Bayesian neural networks \citep{depe_du18}. It can be used for classification and regression, and it is defined as the difference between total uncertainty and aleatoric uncertainty: 
\begin{multline*}
I(\vec{m}) =  \mathrm{H}(\mathbb{E}_{\vec{\theta} \sim p(\vec{\theta} \,|\, \vec{m}(\vec{x}_i; \vec{\Phi}))} \!\left[ p(y \,|\, \vec{\theta})\right] ) \\
- \mathbb{E}_{\vec{\theta} \sim p(\vec{\theta} \,|\, \vec{m}(\vec{x}_i; \vec{\Phi}))} \!\left[ \mathrm{H}(p(y \,|\, \vec{\theta})\right]) \,,
\end{multline*}
with $\mathrm{H}$ being the entropy. Recently, for mutual information,  \citet{wimmer23a} raised the question whether an additive decomposition of total uncertainty in aleatoric and epistemic components makes sense, whereupon a slight modification has been suggested by \citet{schweighofer2023introducing}.

Papers that focus on regression typically use the variance of $\mu$ for quantifying epistemic uncertainty. For NIG distributions the measure becomes: 
$\mathrm{Var}(\mu) = \frac{\beta}{(\alpha-1) \nu} $. 
Some recent works have also investigated variance-based approaches for classification \citep{sale2023secondvar,duan2024evidential}.


 An alternative scalar quantity for epistemic uncertainty is the entropy of $p(\vec{\theta} \,|\, \vec{m}(\vec{x}_i; \vec{\Phi}))$. As far as we know, this measure has only been considered in one recent paper \citep{Kotelevskii2023}, which is surprising, because it is often used as a regularizer --- see Section \ref{subsec:reg_risk_min} for a more in-depth discussion. This last paper is also the only paper that consistently chooses the same measure for regularization and quantification of epistemic uncertainty.  Roughly speaking, one can say that entropy analyzes the uncertainty of $\vec{\theta}$ under $p(\vec{\theta} \,|\, \vec{m}(\vec{x}_i; \vec{\Phi}))$, whereas the covariance matrix rather quantifies the spread around the mean of $\vec{\theta}$. In Appendix~A we provide entropy formulas for NIG, Dirichlet and Gamma distributions. 

Finally, based on distances and reference quantities, another proposal to quantify the different types of uncertainties was recently made \citep{sale2023seconddist}. 
For the concrete choice of the Wasserstein metric, these measures were explicitly determined and calculated for the Dirichlet distribution as an example. 

\section{Theoretical analysis}
\label{sec:theory}
We now return to the two research questions that were formulated in the introduction: 
(a) What properties are needed for second-order risk minimization methods to represent epistemic uncertainty in a quantitatively faithful manner? and (b) Do such methods represent epistemic uncertainty in a faithful manner? The first question will be answered in Section~\ref{subsec:reference_distribution}, and the second one in Sections \ref{subsec:unreg_inner_min}---\ref{subsec:reg_risk_min}. 


\subsection{Reference second-order distribution} \label{subsec:reference_distribution}
Most authors use epistemic uncertainty measures as scoring functions to obtain state-of-the-art performance on downstream tasks such as OOD detection, adversarial robustness and active learning, but they do not intend to further interpret those measures. Yet, there are a few notable exceptions (see Table \ref{tab:overview_contribution}). For the outer maximization case, \citet{bengs2022pitfalls} define two theoretical properties that second-order probability distribution should obey: (1) monotonicity in epistemic uncertainty when the sample size increases, and (2) convergence of the second-order distribution to a Dirac delta function when the sample size grows to infinity. In another paper, \citet{Bengs2023} extend classical properties for proper scoring rules to second-order risk minimization and show that existing second-order losses do not lead to proper scores. For regularized inner loss minimization, \citet{Meinert2023} give theoretical and experimental arguments to advocate that deep evidential regression methods do not represent epistemic uncertainty in a faithful manner. 

The question regarding the ``right'' representation of epistemic uncertainty is a difficult one, akin to the fundamental question of how to do statistical inference in the right way \cite{mart_fu23}. Needless to say, a definite answer to this question cannot be expected. For the purpose of our analysis, we opt for a pragmatic approach specifically tailored to our setting, in which the second-order model is essentially meant to capture the uncertainty of a first-order predictor learned from the training data. The latter is the empirical minimizer of \eqref{eq:exprisktradi}, which is extended to \eqref{eq:exprisk1reg} or \eqref{eq:exprisk2reg} by the second-order learner. The randomness of this minimizer is due to the randomness of the training data $\mathcal{D}$.\footnote{In practice, additional randomness might be caused by the learning algorithm, which might not be guaranteed to deliver the true minimizer.} Theoretically, the uncertainty is therefore reflected by the distribution of the first-order predictor induced by the sampling process that generates the training data. This leads us to the notion of a \emph{reference distribution}, and the idea that the epistemic uncertainty predicted by the second-order learner should ideally be represented by this distribution. 


\begin{definition}
\label{def:faithfulness}
Let $\vec{\theta}_{\mathcal{D}_N}(\vec{x} ;\vec{\Phi}_{\mathcal{D}_N})$ denote the minimizer of \eqref{eq:exprisktradi} for a training set $\mathcal{D}_N$ of size $N$, where $\mathcal{D}_N \sim P^N$. For a given $\vec{x} \in \mathcal{X}$, we define the reference second-order distribution as the conditional distribution induced by $\vec{\theta}_{\mathcal{D}_N}(\vec{x} ;\vec{\Phi}_{\mathcal{D}_N})$:
\begin{align} \label{def:reference_dist}
\begin{split}
    q_N(\vec{\theta} \, | \, \vec{x})  :=& \mathbb{P}\big(\vec{\theta}_{\mathcal{D}_N}(\vec{x} ;\vec{\Phi}_{\mathcal{D}_N}) = \vec{\theta}\big) \\
    =& \int_{(\mathcal{X} \times \mathcal{Y})^N} \! 
\mathbb{I} \big[ \vec{\theta}_{\mathcal{D}_N}(\vec{x} ;\vec{\Phi}_{\mathcal{D}_N}) = \vec{\theta} \big] \, d P^N \, ,
\end{split}
\end{align}
where $\mathbb{I}[ \cdot ]$ is the indicator function. In words, $q_N(\vec{\theta} \, | \, \vec{x})$ is the probability (density) of obtaining the first-order prediction $\vec{\theta}$ when training on a random data set of size $N$ and using the induced model for predicting on $\vec{x}$. 
\end{definition}

Intuitively, this definition follows a frequentist statistics viewpoint, in which a ground-truth model is estimated using limited training data, and epistemic uncertainty is caused by the sum of squared bias and variance of the estimate. The bias-variance decomposition is frequentist in nature, because it assumes a ground-truth model that needs to be estimated using a training dataset of limited size. Typically, the assumption is made that a training dataset of a given size can be resampled to quantify bias (by comparison of the mean to the ground truth) and variance in an objective way.  
More precisely, one can resample the training datasets $d$ times from the underlying distribution, resulting in $d$ empirical risk minimisers $\{\hat{\vec{\theta}}_1, \dots, \hat{\vec{\theta}}_d\}$, which then allow to compute the empirical (cumulative) distribution function: \begin{equation*}
    \hat{Q}_N(\vec{\theta}|\vec{x}) := \frac{1}{d}\sum_{i=1}^d \mathbb{I}(\hat{\vec{\theta}_i}(\vec{x}, \vec{\Phi}) \leq \vec{\theta}).
\end{equation*}  
Moreover, note that the reference second-order distribution depends on various design choices for first-order risk minimization, such as the hypothesis space, loss function and optimization algorithm. As such, we believe that the reference should not necessarily be interpreted as a \emph{ground-truth} epistemic uncertainty representation, but it defines the distribution that might be expected when compatible hypothesis spaces, loss functions and optimization algorithms are used for first-order and second-order risk minimization (i.e.\ for second-order risk minimization one could use the same neural network architecture and optimizer as for first-order risk minimization, and the likelihood of the former might correspond to the conjugate prior of the latter). 
The second-order prediction $p(\vec{\theta} \,| \, \vec{m}(\vec{x};  \vec{\Phi}))$ should ideally be close to the reference $q_N(\vec{\theta} \, | \, \vec{x})$, which can be quantified with any distance metric on probability distributions \citep{rachev2013methods}.

\subsection{Unregularized inner loss minimization}
\label{subsec:unreg_inner_min}
Let us now focus on the second research question. We aim to investigate whether second-order risk minimization methods represent epistemic uncertainty in a faithful manner by analyzing inner and outer loss minimization more in detail. We start by analyzing \eqref{eq:exprisk1reg}, and we first consider the case where the regularization parameter $\lambda$ is zero to simplify the analysis.
Similar to Bayesian model averaging, the two hypothesis spaces $\mathcal{H}_1$ and $\mathcal{H}_2$ can be combined to a joint hypothesis space $\mathcal{H}_{\!J}$ containing $\mathcal{M} \rightarrow \mathbb{P}(\mathcal{Y})$ operators that are defined by
\begin{equation*}
    \label{eq:int}
    p(y \,|\, \vec{m}(\vec{x})) = \mathbb{E}_{\vec{\theta} \sim p(\vec{\theta} \,|\, \vec{m}(\vec{x}; \vec{\Phi}))} \!\left[  p(y \,|\, \vec{\theta}) \right]. \nonumber 
\end{equation*}  
This distribution is known as the predictive distribution in Bayesian statistics, and for exponential family members it can be computed analytically---see Appendix~\ref{sec:proofinject}. As a result, optimization problem \eqref{eq:exprisk1reg} without penalty term can be simplified to
\begin{equation}
    \label{eq:exprisk1}
    \hat{\vec{\Phi}} 
    = \operatorname{arg} \min_{\vec{\Phi}} \, \sum_{i=1}^N L_1(y_i, p(y \,|\, \vec{m}(\vec{x}_i; \vec{\Phi}))) \,.
\end{equation}
The following theorem, in which we analyse the predictive distribution for various types of models, sheds more light on epistemic uncertainty quantification using optimization problem \eqref{eq:exprisk1}. 
\begin{theorem}
\label{thm:exprisk1}
Let $\mathcal{C}_1$ and $\mathcal{C}_J$ be the co-domains of $\mathcal{H}_1$ and $\mathcal{H}_J$. The following properties hold when $\mathcal{H}_2$ consists of
\begin{itemize}
[noitemsep,topsep=0pt,leftmargin=4mm]
\item Dirichlet distributions: $\mathcal{H}_J$ is not injective, and $\mathcal{C}_1 = \mathcal{C}_J$;
\item  NIG distributions:  $\mathcal{H}_J$ is not injective, and $\mathcal{C}_1 \subset \mathcal{C}_J$; 
\item Gamma distributions: $\mathcal{H}_J$ is injective, and $\mathcal{C}_1 \subset \mathcal{C}_J$.
\end{itemize}
\end{theorem}
In the proof (see Appendix~\ref{sec:proofinject}) we have to compute the predictive distribution for the three different cases, 
but let us rather discuss the implications of Theorem~\ref{thm:exprisk1}.
For the first two cases, we find that for all $p(y \,|\, \vec{m})$ in the codomain, multiple $\vec{m}$ exist that generate $p(y \,|\, \vec{m})$.
Consequently, also for the risk minimizer $\bar{\vec{m}}$ induced by $\vec{\Phi}$ in \eqref{eq:exprisk1}, there exists another $\vec{m} \in \mathcal{M}$ that results in the same $p(y \,|\, \vec{m})$.
Thus, the risk minimizer is not uniquely defined. When using the entropy of $p(\vec{\theta} \,|\, \vec{m}(\vec{x}_i; \vec{\Phi}))$ to quantify epistemic uncertainty, one can see that a wide spectrum of values is covered, so different degrees of epistemic uncertainty can be obtained. For NIG distributions, considering second-order risk minimization brings increased flexibility, because $\mathcal{H}_1$ maps Gaussian distributions, whereas $\mathcal{H}_J$ models Student's t-distributions. For  Dirichlet distributions, no flexibility gains are observed, because $\mathcal{H}_1$ and $\mathcal{H}_J$ both map Categorical distributions. Moreover, in Appendix~\ref{sec:convexity} we show that NIG and Dirichlet distributions also differ in a second way: for the latter type of models, the optimization problem is convex when the feature mapping $\vec{\Psi}$ is a linear transformation and $L_1$ the negative log-likelihood, but convexity is not preserved for the former type of models.  

For Gamma distributions, $\mathcal{H}_J$ is injective, so this case again differs from the other two cases. Due to the injectiveness, an identifiability problem does not appear. Second-order risk minimization has as main advantage that a more flexible class of distributions is considered for the data: negative binomial distributions instead of Poisson distributions. However, there is a one-to-one mapping between the parameters of the Gamma distributions and the parameters of the negative binomial distribution. Thus, epistemic uncertainty quantification is also arbitrary in this case, because the risk minimizer yields a Gamma distribution that results in the best negative binomial fit. Consequently, no guarantee can be given that the returned second-order distribution is close to the reference distribution, which was defined in Def.~\ref{def:faithfulness}.



\subsection{Unregularized outer loss minimization}
\label{subsec:unreg_outer_min}
We now analyze optimization problem \eqref{eq:exprisk2reg} with $\lambda=0$: 
\begin{equation}
    \label{eq:exprisk2}
    \hat{\vec{\Phi}} = \operatorname{arg} \min_{\vec{\Phi}} \sum_{i=1}^N \mathbb{E}_{\vec{\theta} \sim p(\vec{\theta} \,|\, \vec{m}(\vec{x}_i; \vec{\Phi}))} \!\left[ L_1(y_i, p(y \,|\, \vec{\theta})) \right] \,.
\end{equation}
Similarly as for inner loss minimization,  closed-form expressions exist for the second-order loss (see Appendix~D and \citet{Charpentier2022NaturalPN}). Yet, it turns out that optimization problem \eqref{eq:exprisk2} has different characteristics than optimization problem \eqref{eq:exprisk1}. Here comes our main result for this subsection:
\begin{theorem}
    \label{thm:exprisk2}
    Let $L_1: \mathcal{Y} \times \mathbb{P}(\mathcal{Y})$
    be convex in its second argument, and let $\mathcal{H}_2$ be a hypothesis space that contains universal approximators, that is, for all $p: \mathcal{X} \times \Theta \rightarrow [0,1]$ and $\epsilon > 0$ there exists $\vec{\Phi}$ such that 
    \begin{equation*}
        \sup_{\vec{\theta} \in \Theta,\, \vec{x} \in \mathcal{X}} \left\| p(\vec{\theta}(\vec{x}))- p(\vec{\theta} \,|\, \vec{m}(\vec{x}_i; \vec{\Phi})) \right\| \leq \epsilon \,.
    \end{equation*}
     A parameter set $\vec{\Phi}$ that yields $ p(\vec{\theta} \,|\, \vec{m}(\vec{x}_i;\hat{\vec{\Phi}})) = \delta(\vec{\theta}(\vec{x}_i))$ for all $i \in \{1,\ldots,N\}$ is a minimizer of \eqref{eq:exprisk2}, where $\delta$ is the Dirac delta function on $\vec{\theta}(\vec{x}_i) \in \Theta$.
\end{theorem}

The proof can be found in Appendix~\ref{sec:proofouter}. The theorem implies that, for universal approximators such as deep neural networks, one of the risk minimizers is a Dirac distribution.
Such a distribution cannot correspond to the reference distribution in Def.~\ref{def:faithfulness}, since a single configuration $\vec{\theta}$ becomes the only candidate to represent $p(y \,|\, \vec{\theta})$.
This is unwanted, because any practical machine learning setting deals with training datasets of limited size. Thus, the epistemic uncertainty should not disappear.

Theorem~\ref{thm:exprisk2} differs from Theorem~\ref{thm:exprisk1} in the sense that non-uniqueness of the risk minimizer is not shown explicitly.
Theorem~\ref{thm:exprisk2} rather indicates that one of the risk minimizers leads to unwanted Dirac distributions, and by means of experiments we will show that such distributions are indeed found with  optimization routines. Let us remark that the theorem generalizes Theorem~1 by \citet{bengs2022pitfalls} in two ways: (1) it analyzes all exponential family members instead of only Dirichlet distributions, which makes the theorem applicable to a wider range of settings,  and (2) it considers universal approximators that map a feature vector to a second-order distribution using $p(\vec{\theta} \,|\, \vec{m}(\vec{x}_i; \vec{\Phi}))$, which makes the theorem practically more relevant.

\begin{figure*}
    \centering
    \includegraphics[width=\textwidth]{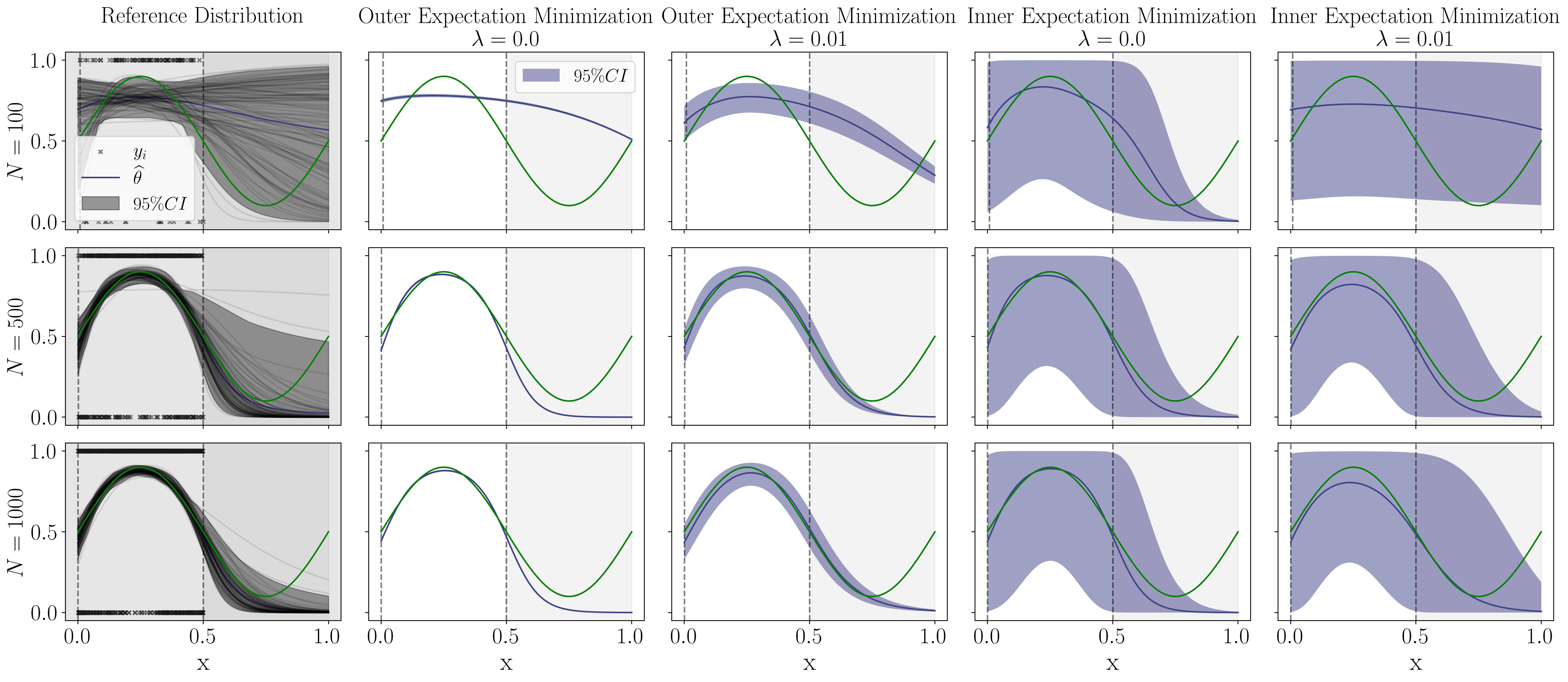}
    \caption{Binary classification experiments for training sample size $N \in \{100, 500, 1000\}$.  The true $\theta$ as a function of $x$ is shown in green. The mean estimated $\theta$ for the reference distribution and the second-order models is given in blue. Confidence bounds of the reference distribution (visualized in grey) are obtained by resampling the training data $100$ times. Confidence bounds for the models trained via second-order risk minimization (visualized in purple) are obtained by the confidence intervals of the learned Beta distribution defined by its $2.5\%$ and $97.5\%$ quantile. The black dots denote one training dataset. See main text for experimental setup. }
    \label{fig:distributions}
\end{figure*}
\begin{figure}
    \centering
\includegraphics[width=.5\textwidth]{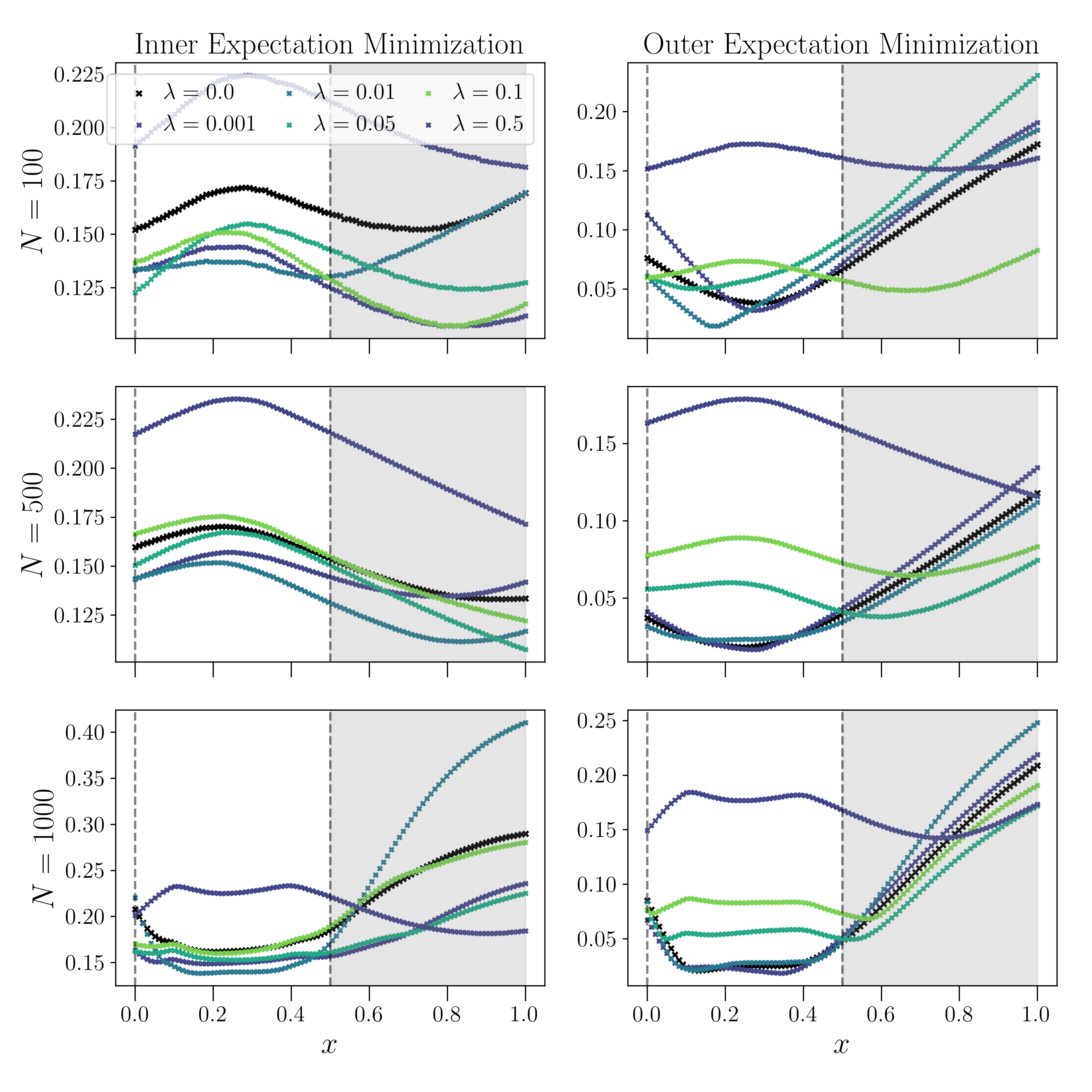}
    \caption{Empirical Wasserstein-$1$ distance between the reference and the estimated second-order distributions for $\lambda \in [0.0, 0.001, 0.01, 0.05, 0.1, 0.5]$. The distance is calculated for $100$ equidistant points in $\mathcal{X}= [0,1]$, by evaluating empirical and second-order distributions in the parameter space and calculating the average $L_1$ distance.}
    \label{fig: distance analysis}
\end{figure}

\subsection{Regularized risk minimization}
\label{subsec:reg_risk_min}
We now take a look at the regularized risk minimization problems that were presented in \eqref{eq:exprisk1reg} and \eqref{eq:exprisk2reg}.
Needless to say, regularization is helpful to remedy some of the issues that were raised so far, but some complications arise, because different authors use different regularizers (see  Table~\ref{tab:methods}). We restrict our analysis to a popular group of regularizers, which enforce $p(\vec{\theta} \,|\, \vec{m}(\vec{x}_i; \vec{\Phi}))$ to look similar to a predefined distribution $p(\vec{\theta} \,|\, \vec{m}_0)$:
\begin{equation}
    \label{eq:kl}
    R(\vec{\Phi}) = \sum_{i=1}^N \mathrm{d}_\mathrm{KL}(p(\vec{\theta} \,|\, \vec{m}(\vec{x}_i; \vec{\Phi})),p(\vec{\theta} \,|\, \vec{m}_0)) \,,
\end{equation}
where $\mathrm{d}_\mathrm{KL}$ denotes the KL-divergence between the two distributions. When $\vec{m}_0$ parameterizes the uniform distribution (on first-order distributions), minimizing KL-divergence corresponds to maximizing the entropy of $p(\vec{\theta} \,|\, \vec{m}(\vec{x}_i; \vec{\Phi}))$, so in fact most regularizers considered in Table~\ref{tab:methods} are considered in our analysis. The main result for regularized second-order loss minimization is presented in the following theorem.  


\begin{theorem}
\label{thm:regrisk}
Consider optimization problems \eqref{eq:exprisk1reg} and \eqref{eq:exprisk2reg} with the regularizer defined by \eqref{eq:kl}. Then, for any data-generating distribution $P$, there exist $\lambda \geq 0$ and $\vec{x} \in \mathcal{X}$ for which the second-order distribution differs from the reference distribution in Def.~\ref{def:faithfulness}. 
\end{theorem}
The proof (see Appendix~\ref{sec:proofregular}) is based on an equivalent formulation of \eqref{eq:exprisk1reg} and \eqref{eq:exprisk2reg} as constrained optimization problems. 
This reformulation reveals the role of the regularizer: it penalizes predictions with a low epistemic uncertainty. More specifically, the estimated epistemic uncertainty for all training data points is summed up, and this sum is fixed to a certain level determined by $\lambda$. In other words, the constraint in the two optimization problems defines an \emph{epistemic uncertainty budget} that cannot be exceeded for the training data points. Now, in light of Def.~\ref{def:faithfulness} we consider for all $\vec{x} \in \mathcal{X}$  a probability distribution defined over the empirical risk minimizer $\hat{\vec{\theta}}$ and epistemic uncertainty quantified by its entropy. If $\lambda$ is chosen sufficiently small, the allowed epistemic uncertainty budget will be too small to represent the epistemic uncertainty of the reference distribution, expressed by entropy or any other epistemic uncertainty measure.     

The main role of the regularizer is to overcome the issues implied by Theorems~\ref{thm:exprisk1} and \ref{thm:exprisk2}. 
Let us take a look first at inner loss minimization. Theorem~\ref{thm:exprisk1} implies that the two terms in \eqref{eq:exprisk1reg} are not conflicting with each other when Dirichlet or NIG distributions are considered for $\mathcal{H}_2$. Among the minimizers of \eqref{eq:exprisk1}, we can simply choose the one that minimizes the regularization term. For Gamma distributions, the situation is slightly different: the two objectives will be conflicting with each other, thus adding a regularization term will yield a higher value for the initial loss function.  
For outer loss minimization, the Dirac delta of Theorem~\ref{thm:exprisk2} will no longer be the risk minimizer. 

\begin{figure*}[t]
    \centering
    \includegraphics[width=\textwidth]{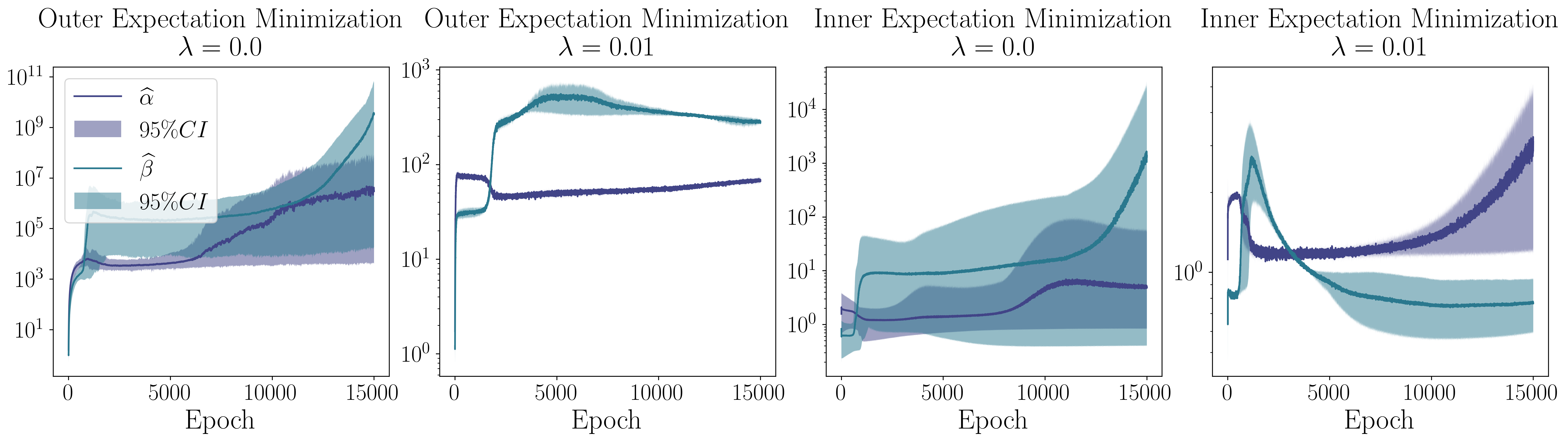}
    \caption{Behavior of the average value of the predicted parameters $\widehat{\alpha}$ and $\widehat{\beta}$, obtained by the second-order learners as a function of the number of training epochs. To obtain the mean and the confidence bounds, $40$ models were trained on the same $N=1000$ instances with different random weight intializations. The average is first taken over the instance space and individual runs are plotted. In addition, the average over different runs is also shown.}
    \label{fig:parameter analysis}
\end{figure*}


\section{Experimental analysis}
\label{sec:experiments}

The main goal of the experiments is to illustrate that our theoretical observations also arise in practice. As argued before, the faithfulness of an epistemic uncertainty representation w.r.t.\ Def.~\ref{def:faithfulness} can only be empirically verified via simulation studies that assume access to the reference distribution $P$.  We performed two simulation studies, covering the classification and the regression case. For the sake of brevity, we only describe the classification case here, but the results for the regression case also confirm the theoretical results (see Appendix~\ref{sec:regresults}). The code for reproducing the experiments is available on GitHub.\footnote{\url{https://github.com/mkjuergens/EpistemicUncertaintyAnalysis}}
\subsection{Setup}
We consider a binary classification setup, because the reference distribution and the second-order fit can be easily visualized in that case.  We generate synthetic training data   
$\mathcal{D}_N = \{(x_i, y_i)\}_{i=1}^N$ of size $N$, with  $x_i \sim \mathrm{Unif}([0,0.5])$, and $y_i \sim \mathrm{Bern}\big(\theta(x_i)\big)$. The Bernoulli parameter $\theta$ is described as a function of the one-dimensional features $x_i$: 
    $\theta(x_i) = 0.5 + 0.4 \, \sin 2\pi x_i$, thus $ \theta(x_i) \in [0.5,0.9]$. 
Intuitively, this looks like a very simple problem, but the aleatoric uncertainty is very high for most $x$, so one needs a lot of training data to estimate $\theta(x)$ with high precision. 

We use the same model architecture for first and second-order risk minimization, consisting of 32 neurons and 2 fully-connected hidden layers. The two approaches only differ in the number of output neurons, which is $1$ for the first-order learner, and $2$ for the second-order learner. The models are trained using the Adam optimizer with a learning rate of $0.0005$ and $5000$ training epochs. For estimating the reference distribution in \eqref{def:reference_dist}, we resample the training set 100 times, estimating each time the first-order risk minimizer $\hat{\theta}_k$ for $k=1, \dots, 100$. In this way the reference distribution of Def.~\ref{def:faithfulness} can be empirically approximated. 

To train the first-order model in \eqref{eq:exprisktradi}, negative log-likelihood is used as a loss function (i.e., logistic loss). The second-order models predict the parameters $\boldsymbol{m}(x) = ( \alpha(x), \beta(x))^\top$ of a Beta distribution and are trained using the loss functions derived from outer and inner expectation minimization, respectively. In both cases, we use closed-form expressions for the loss functions, which avoids the need for sampling---see Appendices \ref{sec:proofinject} and \ref{sec:closedformouter} for derivations. For the regularized models, we use negative entropy of the predicted distribution with $\lambda = 0.01$. 

\subsection{Results}
Figure~\ref{fig:distributions} shows the mean prediction as well as the $95 \%$ confidence levels for the reference model  and the fitted second-order models, trained using different sample sizes. For outer loss minimization without entropy regularization, the conclusions of Theorem~\ref{thm:exprisk2} are nicely reflected in the results: The narrow confidence bounds around the mean prediction indicate a distribution very close to a Dirac delta. When regularization is added, the confidence bounds become wider, but still deviate substantially from those obtained for the reference distribution, especially in the part of the input space where no data has been observed. Hence the epistemic uncertainty quantification by the second-order models is arbitrary and tightly bound to the choice of the regularization strength $\lambda$, and this phenomenon occurs over various sample sizes. For inner loss minimization, a higher epistemic uncertainty is estimated for the example run that is shown, but here, too, the estimated second-order distributions deviate substantially from the reference. However, in accordance with Theorem~\ref{thm:exprisk1}, we would like to emphasize that the parameters of the second-order distribution are unidentifiable for inner loss minimization. For a different run of the optimization algorithm on the same training data, we observed that the results may look very different.

In order to analytically evaluate the distance between the reference distribution and the second-order distributions, the empirical Wasserstein-$1$ distance was calculated per instance and for different $\lambda$ values. Figure \ref{fig: distance analysis} shows the results: Both in regions with and without training data, a relatively high distance to the reference distribution is observed. Using different orders of magnitude for $\lambda$ does not seem to significantly change that pattern.

To analyze some of these observations more deeply, we perform a second experiment in which we train the models with more epochs (up to 15\,000) and 40 runs on the same training dataset. The variation between different runs comes from the uniform initialization of the network parameters. In Figure~\ref{fig:parameter analysis} the estimated second-order parameters $\hat{\alpha}$ and $\hat{\beta}$ as a function of the number of epochs for different runs are displayed. For outer loss minimization, one can see that no convergence occurred after 15\,000 epochs. The estimates of $\alpha$ and $\beta$, averaged over the instance space, keep growing in magnitude, which is in accordance with Theorem~\ref{thm:exprisk2}. When regularization is added, $\alpha$ and $\beta$ grow in a much smaller rate, thereby also solving the convergence issues, as expected due to Theorem~\ref{thm:regrisk}. For inner loss minimization, one can observe large standard deviations in the estimates for $\alpha$ and $\beta$ when the Adam optimizer is running multiple times on the same training dataset. This is expected because of Theorem~\ref{thm:exprisk1}. Adding some regularization substantially decreases this variation.  

\section{Conclusion}
In this paper, we proposed the comparison with a reference distribution as an objective procedure to assess whether evidential deep learning methods represent epistemic uncertainty in a faithful manner. Using this procedure, we showed that epistemic uncertainty is in general not faithfully represented. More specifically, the regularizers that characterize evidential deep learning methods are mainly needed to stabilize the corresponding optimization problems, thereby yielding an "uncertainty budget" that cannot be exceeded. Conversely, in first-order risk minimization, regularization is mainly used to prevent overfitting by controlling the bias-variance trade-off. The resulting measures of epistemic uncertainty cannot be interpreted in a quantitative manner. In many downstream tasks, such as out-of-distribution detection, only relative interpretations of epistemic uncertainty are needed. This probably explains the strong performance of evidential deep learning methods on such tasks. We hypothesize that the OOD scores reported by evidential deep learning methods could be interpreted as density estimates in feature space. This observation was made by \citet{Grathwohl2020,Charpentier2022NaturalPN}, but we believe that more research is needed to draw final conclusions of that kind.

\newpage
\section*{Impact statement}
Machine learning systems are increasingly prevalent in high-stake application domains with a direct impact on humans, such as medical diagnostics, legal decisions, self-driving cars, fraud detection, natural disaster forecasting, etc. The profound consequences of inaccurate predictions in these areas emphasize the necessity of assessing the certainty of machine-learned forecasts. The current upswing in research on uncertainty quantification in the machine learning community underscores the recognition of this crucial aspect.

Taking center stage is Evidential Deep Learning (EDL), a method gaining prominence in safety-critical domains. However, despite its growing utilization, the operational principles of EDL remain less than fully understood. Rather than solely focusing on developing superior techniques for specific datasets, we advocate for a meticulous examination of existing methods. Unraveling their limitations is not just an academic exercise, but it serves as a practical guide for decision-makers. Understanding when to employ these methods and when to exercise caution is paramount in navigating the intricate terrain of high-stakes predictions.

\section*{Acknowledgements}
M.J. and W.W. received funding from the Flemish Government under the “Onderzoeksprogramma Artifici\"ele Intelligentie (AI) Vlaanderen” programme.
\bibliographystyle{unsrtnat}
\bibliography{ref}

\begin{thebibliography}{39}
\providecommand{\natexlab}[1]{#1}
\providecommand{\url}[1]{\texttt{#1}}
\expandafter\ifx\csname urlstyle\endcsname\relax
  \providecommand{\doi}[1]{doi: #1}\else
  \providecommand{\doi}{doi: \begingroup \urlstyle{rm}\Url}\fi

\bibitem[Kendall and Gal(2017)]{kend_wu17}
A.\ Kendall and Y.\ Gal.
\newblock What uncertainties do we need in {B}ayesian deep learning for computer vision?
\newblock In \emph{Proc.\ NIPS, 30th Advances in Neural Information Processing Systems}, pages 5574--5584, 2017.

\bibitem[H\"ullermeier and Waegeman(2021)]{mpub440}
E.~H\"ullermeier and W.\ Waegeman.
\newblock Aleatoric and epistemic uncertainty in machine learning: {A}n introduction to concepts and methods.
\newblock \emph{Machine Learning}, 110\penalty0 (3):\penalty0 457--506, 2021.
\newblock \doi{10.1007/s10994-021-05946-3}.

\bibitem[Depeweg et~al.(2018)Depeweg, Hernandez-Lobato, Doshi-Velez, and Udluft]{depe_du18}
S.\ Depeweg, J.M.\ Hernandez-Lobato, F.\ Doshi-Velez, and S.\ Udluft.
\newblock Decomposition of uncertainty in {B}ayesian deep learning for efficient and risk-sensitive learning.
\newblock In \emph{Proc.\ ICML, 35th International Conference on Machine Learning}, pages 1184--1193, 2018.

\bibitem[Graves(2011)]{Graves2011}
Alex Graves.
\newblock Practical variational inference for neural networks.
\newblock In \emph{Proc.\ {NIPS}, 24th Advances in Neural Information Processing Systems}, page 2348–2356, 2011.

\bibitem[Sensoy et~al.(2018)Sensoy, Kaplan, and Kandemir]{sens_ed18}
M.\ Sensoy, L.\ Kaplan, and M.\ Kandemir.
\newblock Evidential deep learning to quantify classification uncertainty.
\newblock In \emph{Proc.\ NeurIPS, 31st Conference on Neural Information Processing Systems}, pages 3183--3193, 2018.

\bibitem[Malinin and Gales(2018)]{MalininG18}
A.~Malinin and M.~Gales.
\newblock Predictive uncertainty estimation via prior networks.
\newblock In \emph{Proc.\ {NeurIPS}, 31st Advances in Neural Information Processing Systems}, pages 7047--7058, 2018.

\bibitem[Malinin and Gales(2019)]{MalininG19}
A.~Malinin and M.~Gales.
\newblock Reverse {KL}-divergence training of prior networks: Improved uncertainty and adversarial robustness.
\newblock In \emph{Proc.\ {NeurIPS}, 32nd Advances in Neural Information Processing Systems}, pages 14520--14531, 2019.

\bibitem[Charpentier et~al.(2020)Charpentier, Z\"ugner, and G\"unnemann]{char_ue20}
B.~Charpentier, D.~Z\"ugner, and S.~G\"unnemann.
\newblock Posterior network: Uncertainty estimation without {OOD} samples via density-based pseudo-counts.
\newblock In \emph{Proc.\ NeurIPS, 33rd Neural Information Processing Systems}, volume~33, pages 1356--1367, 2020.

\bibitem[Huseljic et~al.(2020)Huseljic, Sick, Herde, and Kottke]{HuseljicSHK20}
D.~Huseljic, B.~Sick, M.~Herde, and D.~Kottke.
\newblock Separation of aleatoric and epistemic uncertainty in deterministic deep neural networks.
\newblock In \emph{Proc.\ {ICPR}, 25th International Conference on Pattern Recognition}, pages 9172--9179. {IEEE}, 2020.

\bibitem[Malinin et~al.(2020)Malinin, Chervontsev, Provilkov, and Gales]{MalininUnpub}
Andrey Malinin, Sergey Chervontsev, Ivan Provilkov, and Mark Gales.
\newblock Regression prior networks.
\newblock \emph{arXiv preprint arXiv:2006.11590}, 2020.

\bibitem[Amini et~al.(2020)Amini, Schwarting, Soleimany, and Rus]{amini2020deep}
Alexander Amini, Wilko Schwarting, Ava Soleimany, and Daniela Rus.
\newblock Deep evidential regression.
\newblock In \emph{Proc.\ {NeurIPS}, 33rd Advances in Neural Information Processing Systems}, volume~33, pages 14927--14937, 2020.

\bibitem[Ma et~al.(2021)Ma, Han, Zhang, Fu, Zhou, and Hu]{Ma2021TrustworthyMR}
Huan Ma, Zongbo Han, Changqing Zhang, H.~Fu, Joey~Tianyi Zhou, and Qinghua Hu.
\newblock Trustworthy multimodal regression with mixture of normal-inverse gamma distributions.
\newblock In \emph{Proc.\ {NeurIPS}, 34th Advances in Neural Information Processing Systems}, pages 6881--6893, 2021.

\bibitem[Tsiligkaridis(2021)]{tsiligkaridis2021information}
Theodoros Tsiligkaridis.
\newblock Information aware max-norm {D}irichlet networks for predictive uncertainty estimation.
\newblock \emph{Neural Networks}, 135:\penalty0 105--114, 2021.

\bibitem[Kopetzki et~al.(2021)Kopetzki, Charpentier, Z{\"{u}}gner, Giri, and G{\"{u}}nnemann]{KopetzkiCZGG21}
A.{-}Kathrin Kopetzki, B.~Charpentier, D.~Z{\"{u}}gner, S.~Giri, and S.~G{\"{u}}nnemann.
\newblock Evaluating robustness of predictive uncertainty estimation: Are {D}irichlet-based models reliable?
\newblock In \emph{Proc.\ {ICML}, 38th International Conference on Machine Learning}, pages 5707--5718, 2021.

\bibitem[Bao et~al.(2021)Bao, Yu, and Kong]{Bao2021}
Wentao Bao, Qi~Yu, and Yu~Kong.
\newblock Evidential deep learning for open set action recognition.
\newblock \emph{2021 IEEE/CVF International Conference on Computer Vision (ICCV)}, pages 13329--13338, 2021.

\bibitem[Liu et~al.(2021)Liu, Amini, Zhu, Karaman, Han, and Rus]{Liu2021}
Zhijian Liu, Alexander Amini, Sibo Zhu, Sertac Karaman, Song Han, and Daniela~L. Rus.
\newblock Efficient and robust lidar-based end-to-end navigation.
\newblock In \emph{2021 IEEE International Conference on Robotics and Automation (ICRA)}, page 13247–13254. IEEE Press, 2021.

\bibitem[Charpentier et~al.(2022)Charpentier, Borchert, Zugner, Geisler, and G\"unnemann]{Charpentier2022NaturalPN}
Bertrand Charpentier, Oliver Borchert, Daniel Zugner, Simon Geisler, and Stephan G\"unnemann.
\newblock Natural posterior network: Deep bayesian predictive uncertainty for exponential family distributions.
\newblock In \emph{International Conference on Learning Representations}, 2022.

\bibitem[Oh and Shin(2022)]{Oh_Shin_2022}
Dongpin Oh and Bonggun Shin.
\newblock Improving evidential deep learning via multi-task learning.
\newblock In \emph{Proc.\ {AAAI}, 36th Conference on Artificial Intelligence}, number~7, pages 7895--7903, 2022.

\bibitem[Pandey and Yu(2023)]{Shankar2023AAAI}
Deep~Shankar Pandey and Qi~Yu.
\newblock Evidential conditional neural processes.
\newblock In \emph{Proc.\ {AAAI}, 37th Conference on Artificial Intelligence}, pages 9389--9397, 2023.

\bibitem[Kotelevskii et~al.(2023)Kotelevskii, Horváth, Nandakumar, Takáč, and Panov]{Kotelevskii2023}
Nikita Kotelevskii, Samuel Horváth, Karthik Nandakumar, Martin Takáč, and Maxim Panov.
\newblock Dirichlet-based uncertainty quantification for personalized federated learning with improved posterior networks.
\newblock \emph{arXiv preprint arXiv:2312.11230}, 2023.

\bibitem[Bengs et~al.(2022)Bengs, H{\"u}llermeier, and Waegeman]{bengs2022pitfalls}
Viktor Bengs, Eyke H{\"u}llermeier, and Willem Waegeman.
\newblock Pitfalls of epistemic uncertainty quantification through loss minimisation.
\newblock In \emph{Proc.\ {NeurIPS}, 35th Advances in Neural Information Processing Systems}, 2022.

\bibitem[Meinert et~al.(2023)Meinert, Gawlikowski, and Lavin]{Meinert2023}
Nis Meinert, Jakob Gawlikowski, and Alexander Lavin.
\newblock The unreasonable effectiveness of deep evidential regression.
\newblock In \emph{Proc.\ {AAAI}, 37th Conference on Artificial Intelligence}, pages 9134--9142, 2023.

\bibitem[Bengs et~al.(2023)Bengs, H\"{u}llermeier, and Waegeman]{Bengs2023}
Viktor Bengs, Eyke H\"{u}llermeier, and Willem Waegeman.
\newblock On second-order scoring rules for epistemic uncertainty quantification.
\newblock In \emph{Proc.\ {ICML}, 40th International Conference on Machine Learning}, 2023.

\bibitem[Yang et~al.(2022)Yang, Wang, Zou, Zhou, Ding, PENG, Wang, Chen, Li, Sun, Du, Zhou, Zhang, Hendrycks, Li, and Liu]{Yang2022}
Jingkang Yang, Pengyun Wang, Dejian Zou, Zitang Zhou, Kunyuan Ding, WENXUAN PENG, Haoqi Wang, Guangyao Chen, Bo~Li, Yiyou Sun, Xuefeng Du, Kaiyang Zhou, Wayne Zhang, Dan Hendrycks, Yixuan Li, and Ziwei Liu.
\newblock Open{OOD}: Benchmarking generalized out-of-distribution detection.
\newblock In \emph{Proc.\ {NeurIPS}, 35th Advances in Neural Information Processing Systems}, pages 32598--32611, 2022.

\bibitem[Park et~al.(2023)Park, Choi, Kim, Han, and Moon]{Park2023}
Younghyun Park, Wonjeong Choi, Soyeong Kim, Dong-Jun Han, and Jaekyun Moon.
\newblock Active learning for object detection with evidential deep learning and hierarchical uncertainty aggregation.
\newblock In \emph{International Conference on Learning Representations}, 2023.

\bibitem[Ulmer et~al.(2023)Ulmer, Hardmeier, and Frellsen]{Ulmer2023}
Dennis Ulmer, Christian Hardmeier, and Jes Frellsen.
\newblock Prior and posterior networks: A survey on evidential deep learning methods for uncertainty estimation.
\newblock \emph{Transactions of Machine Learning Research}, 2023.

\bibitem[Bissiri et~al.(2016)Bissiri, Holmes, and Walker]{biss_ag16}
P.G. Bissiri, C.C. Holmes, and S.G. Walker.
\newblock A general framework for updating belief distributions.
\newblock \emph{Journal of the Royal Statistical Society: Series B (Statistical Methodology)}, 78\penalty0 (5):\penalty0 1103--1130, 2016.

\bibitem[Meinert and Lavin(2022)]{Meinert2022}
Nis Meinert and Alexander Lavin.
\newblock Multivariate deep evidential regression.
\newblock \emph{arXiv preprint arXiv:2104.06135}, 2022.

\bibitem[Bramlage et~al.(2023)Bramlage, Karg, and Curio]{Bramlage2023}
Lennart Bramlage, Michelle Karg, and Crist\'obal Curio.
\newblock Plausible uncertainties for human pose regression.
\newblock In \emph{Proceedings of the IEEE/CVF International Conference on Computer Vision (ICCV)}, pages 15133--15142, 2023.

\bibitem[Gelman et~al.(2013)Gelman, Carlin, Stern, Dunson, Vehtari, and Rubin]{Gelman2013}
A.~Gelman, J.B. Carlin, H.S. Stern, D.B. Dunson, A.~Vehtari, and D.B. Rubin.
\newblock \emph{Bayesian Data Analysis, Third Edition}.
\newblock Chapman \& Hall/CRC Texts in Statistical Science. Taylor \& Francis, 2013.
\newblock ISBN 9781439840955.

\bibitem[Wimmer et~al.(2023)Wimmer, Sale, Hofman, Bischl, and H\"ullermeier]{wimmer23a}
Lisa Wimmer, Yusuf Sale, Paul Hofman, Bernd Bischl, and Eyke H\"ullermeier.
\newblock Quantifying aleatoric and epistemic uncertainty in machine learning: Are conditional entropy and mutual information appropriate measures?
\newblock In \emph{Proc.\ UAI, 39th Conference on Uncertainty in Artificial Intelligence}, volume 216 of \emph{Proceedings of Machine Learning Research}, pages 2282--2292, 2023.

\bibitem[Schweighofer et~al.(2023)Schweighofer, Aichberger, Ielanskyi, and Hochreiter]{schweighofer2023introducing}
Kajetan Schweighofer, Lukas Aichberger, Mykyta Ielanskyi, and Sepp Hochreiter.
\newblock Introducing an improved information-theoretic measure of predictive uncertainty.
\newblock In \emph{NeurIPS 2023 Workshop on Mathematics of Modern Machine Learning}, 2023.

\bibitem[Sale et~al.(2023{\natexlab{a}})Sale, Hofman, Wimmer, H{\"u}llermeier, and Nagler]{sale2023secondvar}
Yusuf Sale, Paul Hofman, Lisa Wimmer, Eyke H{\"u}llermeier, and Thomas Nagler.
\newblock Second-order uncertainty quantification: Variance-based measures.
\newblock \emph{arXiv preprint arXiv:2401.00276}, 2023{\natexlab{a}}.

\bibitem[Duan et~al.(2024)Duan, Caffo, Bai, Sair, and Jones]{duan2024evidential}
Ruxiao Duan, Brian Caffo, Harrison~X Bai, Haris~I Sair, and Craig Jones.
\newblock Evidential uncertainty quantification: A variance-based perspective.
\newblock In \emph{Proceedings of the IEEE/CVF Winter Conference on Applications of Computer Vision}, pages 2132--2141, 2024.

\bibitem[Sale et~al.(2023{\natexlab{b}})Sale, Bengs, Caprio, and H{\"u}llermeier]{sale2023seconddist}
Yusuf Sale, Viktor Bengs, Michele Caprio, and Eyke H{\"u}llermeier.
\newblock Second-order uncertainty quantification: A distance-based approach.
\newblock \emph{arXiv preprint arXiv:2312.00995}, 2023{\natexlab{b}}.

\bibitem[Martin(2023)]{mart_fu23}
R.~Martin.
\newblock Fisher's underworld and the behavioral-statistical reliability balance in scientific inference.
\newblock \emph{arXiv preprint arXiv:2312.14912}, 2023.

\bibitem[Rachev et~al.(2013)Rachev, Klebanov, Stoyanov, and Fabozzi]{rachev2013methods}
Svetlozar~T Rachev, Lev~B Klebanov, Stoyan~V Stoyanov, and Frank Fabozzi.
\newblock \emph{The Methods of Distances in the Theory of Probability and Statistics}, volume~10.
\newblock Springer, 2013.

\bibitem[Grathwohl et~al.(2020)Grathwohl, Wang, Jacobsen, Duvenaud, Norouzi, and Swersky]{Grathwohl2020}
Will Grathwohl, Kuan-Chieh~Jackson Wang, J{\"o}rn-Henrik Jacobsen, David~Kristjanson Duvenaud, Mohammad Norouzi, and Kevin Swersky.
\newblock Your classifier is secretly an energy based model, and you should treat it like one.
\newblock In \emph{International Conference on Learning Representations}, 2020.

\bibitem[Boyd and Vandenberghe(2004)]{Boyd2004}
Stephen Boyd and Lieven Vandenberghe.
\newblock \emph{Convex Optimization}.
\newblock {Cambridge University Press}, March 2004.
\newblock ISBN 0521833787.

\end{thebibliography}

\onecolumn
\appendix
\section{Basic properties of exponential family members}
\label{sec:basic_properties}
In this appendix we review some important formulas for exponential family members. 

\subsection*{Dirichlet distribution}
Let $\vec{\theta} = (\theta_1,...,\theta_K)$ be the parameters of the Categorical distribution and let $\vec{m}  = (m_1,\ldots,m_K) \in \mathbb{R}_+^K $ be the parameters of the Dirichlet distribution with probability density function:
\begin{equation}
    \label{eq:dir}
    p(\vec{\theta} \,|\, \vec{m}) = \frac{1}{B(\vec{m})} \prod_{k=1}^K \theta_k^{m_k-1}
\end{equation}
with $B(\vec{m})$ the Beta function, $
    \vec{\theta} = (\theta_1, \ldots , \theta_K) \in \triangle^K$ and $\triangle^K$ the probability $(K-1)$-simplex. 

The entropy of the Dirichlet distribution is given by:
$$\mathrm{H}(p(\vec{\theta} \,|\, \vec{m})) = \log B(\vec{m}) +(m_0-K)\psi(m_0) - \sum_{k=1}^K (m_k-K)\psi(m_k) \,,$$
with $\psi$ the digamma function and $m_0 = \sum_{k=1}^K m_k$. 

\subsection*{Normal-inverse Gamma distribution}
Let $\vec{\theta} = (\mu,\sigma^2)$ be the parameters of the Gaussian distribution and let $\vec{m} = (\gamma,\nu,\alpha,\beta) \in \mathbb{R} \times \mathbb{R}_+ \times [1,\infty) \times \mathbb{R}_+$ be the parameters of the normal-inverse-Gamma (NIG) distribution with probability density function:
\begin{equation}
    \label{eq:der}
    p(\mu,\sigma^2 \,|\, \gamma, \nu, \alpha, \beta) = \frac{\sqrt{\nu}}{\sqrt{2 \pi \sigma^2}} \frac{\beta^{\alpha}}{\Gamma(\alpha)} \left( \frac{1}{\sigma^2} \right)^{(\alpha + 1)} \exp \left\{ -\frac{2 \beta + \nu (\mu - \gamma)^2}{2 \sigma^2} \right\} \,,
\end{equation}
The entropy of the normal-inverse-Gamma distribution is given by:
$$\mathrm{H}(p(\mu,\sigma^2 \,|\, \gamma, \nu, \alpha, \beta)) = \frac{1}{2} + \log \left( \sqrt{2 \pi} \beta^{\frac{3}{2}} \Gamma(\alpha)\right) - \frac{\log \nu}{2}  +\alpha - \left(\alpha + \frac{3}{2} \right) \psi(\alpha) \,,$$
with $\Gamma$ the Gamma function. 	 
		
\subsection*{Gamma distribution}
Let $\theta$ be the parameter of a Poisson distribution and let $\vec{m}=(\alpha,\beta) \in \mathbb{R}_+^2$ be the shape and rate parameters of the Gamma distribution with probability density function:
$$p(\theta \,|\, \alpha,\beta) = \frac{\beta^{\alpha}}{\Gamma(\alpha)} \theta^{\alpha-1} e^{-\beta \theta } \,. $$
The entropy of the Gamma distribution is given by: 
$$\mathrm{H}(p(\theta \,|\, \alpha,\beta)) 
= \alpha - \ln \beta + \ln \Gamma(\alpha) + (1-\alpha) \psi(\alpha) \,.$$

\section{Proof of Theorem~\ref{thm:exprisk1}}
\label{sec:proofinject}
In this proof we derive 
\begin{eqnarray}
    \label{eq:intBis}
    p(y \,|\, \vec{m}(\vec{x})) &=& \mathbb{E}_{\vec{\theta} \sim p(\vec{\theta} \,|\, \vec{m}(\vec{x}; \vec{\Phi}))} \, p(y \,|\, \vec{\theta})
		= \int \! p(y \,|\, \vec{\theta}) \, p(\vec{\theta} \,|\, \vec{m}(\vec{x})) \, d \vec{\theta}
\end{eqnarray}  
for different members of the exponential family. 

When Dirichlet distributions are chosen for $\mathcal{H}_2$, the predictive distribution is given by\footnote{For this derivation we were inspired by https://stephentu.github.io/writeups/dirichlet-conjugate-prior.pdf}: 
\begin{eqnarray*}
    \label{eq:preddiri}
    p(y =k \,|\, \vec{m} = (m_1,...,m_k)) 
    &=& \int_{\triangle^K} \! p(y =k \,|\, \vec{\theta}) \, p(\vec{\theta} \,|\, \vec{m}) \, d \vec{\theta}\\
    &=& \int_{\triangle^K} \! \frac{\Gamma(\sum_{j=1}^K m_j)}{\prod_{j=1}^K \Gamma(m_j)} \prod_{j=1}^K \theta_j^{m_j-1} \theta_k \, d \vec{\theta}\\
    &=& \frac{\Gamma(\sum_{j=1}^K m_j)}{\prod_{j=1}^K \Gamma(m_j)} \int_{\triangle^K} \! \prod_{j=1}^K \theta_j^{m_j-1} \theta_k \, d \vec{\theta} \\  
     &=& \frac{\Gamma(\sum_{j=1}^K m_j)}{\prod_{j=1}^K \Gamma(m_j)} \int_{\triangle^K} \! \mathrm{d} \vec{\theta} \,   \prod_{j=1}^K \theta_j^{\mathbb{I}(j=k) + m_j-1}  \\  
       &=& \frac{\Gamma(\sum_{j=1}^K m_j)}{\prod_{j=1}^K \Gamma(m_j)} \frac{\prod_{j=1}^K \Gamma(\mathbb{I}(j=k) + m_j)}{\Gamma(1+ \sum_{j=1}^K m_j)}   \\  
       &=& \frac{m_k}{\sum_{j=1}^K m_j} = \vartheta_k \,.
\end{eqnarray*}
This is the probability mass function of the Categorical distribution with parameters $\vartheta_1,\vartheta_2,...,\vartheta_k$. In this case $\mathcal{H}_J$ does not contain injective $\mathcal{M} \rightarrow \mathbb{P} ( \mathcal{Y})$ operators, and $\mathcal{C}_1 = \mathcal{C}_J$ because the two hypothesis spaces contain Categorical distributions. 
		
When normal-inverse Gamma distributions are chosen for $\mathcal{H}_2$, the predictive distribution is derived by \citet{amini2020deep}:
\begin{eqnarray*}
    \label{eq:preddiri}
    p(y \,|\, \vec{m} = (\gamma,\nu,\alpha,\beta)) 
    &=& \int \! p(y \,|\, \vec{\theta}) \, p(\vec{\theta} \,|\, \vec{m}) \, d \vec{\theta} \\
    &=& \int_{\sigma^2 = 0}^{\infty} \int_{\mu=-\infty}^{\infty} p(y \,|\, \mu, \sigma^2) \, p(\mu,\sigma^2 \,|\, \gamma,\nu,\alpha,\beta) \, d \sigma^2 \, d \mu \\ 
    &=& \cdots \\
    &=& \frac{\Gamma(1/2 + \alpha)}{\Gamma(\alpha)} \sqrt{\frac{\nu}{\pi}} (2 \beta (1+\nu))^\alpha (\nu (y-\gamma^2) + 2\beta (1+\nu))^{(-\frac{1}{2} + \alpha)} \\
    &=& \operatorname{St}\!\left( y; \, \gamma,\frac{\beta (1+\nu)}{\nu \alpha}, 2 \alpha \right) \,,
\end{eqnarray*}
where $\operatorname{St}(y; \mu_\mathrm{St},\sigma^2_\mathrm{St}, \nu_\mathrm{St})$ stands for Student's-$t$ distribution with location $\mu_\mathrm{St}$, scale $\sigma^2_\mathrm{St}$ and $\nu_\mathrm{St}$ degrees of freedom. $\mathcal{H}_J$ does not contain injective $\mathcal{M} \rightarrow \mathbb{P} ( \mathcal{Y})$ operators, and $\mathcal{C}_1 \subset  \mathcal{C}_J$, because the Gaussian distribution is a special case of the Student's-$t$ distribution. 

When Gamma distributions are chosen for $\mathcal{H}_2$, the predictive distribution is given by\footnote{For this derivation we were inspired by https://people.stat.sc.edu/Hitchcock/stat535slidesday18.pdf}:
\begin{eqnarray*}
    \label{eq:predGamma}
    p(y =k \,|\, \vec{m} = (\alpha,\beta)) 
    &=& \int \! p(y =k \,|\, \theta) \, p(\theta \,|\, \vec{m}) \, d \theta \\
    &=& \int \! \left[ \frac{\theta^k e^{-\theta}}{k!} \right] \left[ \frac{\beta^{\alpha}}{\Gamma(\alpha)} \theta^{\alpha-1} e^{-\beta\theta}\right] \, d \theta \\
     &=& \frac{\beta^{\alpha}}{ \Gamma(k+1) \, \Gamma(\alpha)}\int \! \theta^{\alpha-1 +k} e^{-(\beta+1)\theta} \, d \theta\\
     &=& \frac{\beta^{\alpha}}{ \Gamma(k+1) \, \Gamma(\alpha)}  \frac{\Gamma(k+\alpha)}{(\beta+1)^{k+\alpha}} \\
     &=& \frac{ \Gamma(k+\alpha)}{ \Gamma(k+1) \, \Gamma(\alpha)}  \left[\frac{\beta}{\beta+1}\right]^{\alpha} \left[ \frac{1}{\beta+1} \right]^k  \\
     &=& \binom{k+\alpha-1}{k}  \left[\frac{\beta}{\beta+1}\right]^{\alpha} \left[ \frac{1}{\beta+1} \right]^k \,.  
\end{eqnarray*}
This is the probability mass function of a negative binomial distribution with number of successes $r=\alpha$ and success probability $p= \beta / (\beta+1)$. 
The p.m.f.\ is an injective $\mathcal{M} \rightarrow \mathbb{P} ( \mathcal{Y})$ operator, and $\mathcal{C}_1 \subset  \mathcal{C}_J$ because the Poisson distribution is a special case of the negative binomial distribution.

\section{Additional theoretical results for inner loss minimization}
\label{sec:convexity}
In this appendix we further analyze the convexity properties of inner loss minimization without regularization, using the negative log-likelihood as first-order loss function. Let us first recall the standard definition of convexity.
\begin{definition}
    A function $f: \mathbb{R}^D \rightarrow \mathbb{R}$ is convex if 
    \begin{equation}
        \label{eq:convex}
        f(\rho \vec{x} + (1-\rho) \vec{z}) \leq \rho f(\vec{x}) + (1-\rho) f(\vec{z}) 
    \end{equation}
    for all $\vec{x}, \vec{z}$ in the domain of $f$ and all $\rho \in [0,1]$.
    Furthermore, if \eqref{eq:convex} holds as a strict inequality for all $\vec{x}, \vec{z}$ in the domain of $f$ and all $\rho \in [0,1]$, then $f$ is called strictly convex. 
\end{definition}

The following theorem presents formal results for inner expectation maximization using linear models as underlying model class for the neural network. 

\begin{theorem}
    \label{thm:exprisk1bis}
    The following statements hold concerning the convexity of the risk minimization problem in \eqref{eq:exprisk1}. 
    \begin{itemize}
        \item Let $\mathcal{H}_2$ consist of Dirichlet distributions, let $L_1$ be the negative log-likelihood, and let $\vec{\Psi}$ be a linear transformation, then \eqref{eq:exprisk1} is convex.
        \item Let $\mathcal{H}_2$ consist of normal-inverse Gamma distributions, let $L_1$ be the negative log-likelihood, and let $\vec{\Psi}$ be a linear transformation, then \eqref{eq:exprisk1} is not convex.
    \end{itemize}
\end{theorem}
\begin{proof}
    Let us start with proving the first statement.
    With Dirichlet distributions for $\mathcal{H}_2$ and $\vec{\Psi}$ a linear transformation, \eqref{eq:int} becomes
    \begin{equation}
        \label{eq:logreg}
        p(y=k \,|\, \vec{m}) = \frac{\mathrm{e}^{\vec{w}_k^\top \vec{\Psi}(\vec{x})}}{\sum_{j=1}^K \mathrm{e}^{\vec{w}_j^\top \vec{\Psi}(\vec{x})}} \,.
    \end{equation}
    This is nothing more than the class of functions considered in multinomial logistic regression.
    The per-instance risk becomes
    \begin{equation*}
        L_1(y, p(y \,|\, \vec{m}(\vec{x}_i))) = - \log \left(\frac{\mathrm{e}^{\vec{w}_y^\top \vec{\Psi}(\vec{x})}}{\sum_{j=1}^K \mathrm{e}^{\vec{w}_j^\top \vec{\Psi}(\vec{x})}} \right) = \log \left(\sum_{j=1}^K \mathrm{e}^{\vec{w}_j^\top \vec{\Psi}(\vec{x})}\right) - \vec{w}_y^\top \vec{\Psi}(\vec{x}) \,.
    \end{equation*}
    A function of type log-sum-exponential is a convex function~\citep{Boyd2004}.
    The second part of the risk is linear, so the two components together are convex.

    Let us now prove the second statement by evaluating~\eqref{eq:convex} for
    \begin{equation*}
        f(\gamma, \nu, \alpha, \beta) = -\log \operatorname{St}(y; \gamma, \sigma^2_\mathrm{St}(\nu, \alpha, \beta), 2 \alpha).
    \end{equation*}
    Since $\exists \,\{(\nu, \beta), (\nu', \beta')\}$ with $(\nu, \beta) \neq (\nu', \beta')$ and $\sigma^2_\mathrm{St}(\nu, \alpha, \beta) = \sigma^2_\mathrm{St}(\nu', \alpha, \beta')$ $\forall \alpha$ for these points the r.h.s.\ of~\eqref{eq:convex} can be written as:
    \begin{equation*}
        \rho f(\gamma, \nu, \alpha, \beta) + (1 - \rho) f(\gamma, \nu', \alpha, \beta') = -\log \operatorname{St}(y; \gamma, \sigma^2_\mathrm{St}(\nu, \alpha, \beta), 2 \alpha)\,.
    \end{equation*}
    Next, we note that $\hat\sigma^2_\mathrm{St} = \operatorname{arg} \max_{\sigma^2_\mathrm{St}} \operatorname{St}(y; \mu_\mathrm{St}, \sigma^2_\mathrm{St}, \nu_\mathrm{St}) = (x - \mu_\mathrm{St})^2$.
    Let $\hat\sigma^2_\mathrm{St} = \sigma^2_\mathrm{St}(\hat\nu, \hat\alpha, \hat\beta) = \sigma^2_\mathrm{St}(\hat\nu', \hat\alpha, \hat\beta')$,
    then the l.h.s.\ of~\eqref{eq:convex} is always greater or equal than the r.h.s.\ for these points,
    \begin{equation*}
        -\log \operatorname{St}\!\left( y; \gamma, \sigma^2_\mathrm{St}(\rho \hat\nu + (1 - \rho) \hat\nu', \hat\alpha, \rho \hat\beta + (1 - \rho) \hat\beta'), 2 \hat\alpha \right) \ge -\log \operatorname{St}\!\left( y; \gamma, \hat\sigma^2_\mathrm{St}, 2 \hat\alpha \right).
    \end{equation*}
    If $\sigma^2_\mathrm{St}$ is an affine function both sides are equal. However, if $\vec{\Psi}$ is a linear transformation and $(\gamma, \nu, \alpha, \beta)$ are defined as output neurons with linear and exponential activations (see Section~2) or their related formulations with the softplus function, the inequality becomes strict.
    Since~\eqref{eq:convex} has to hold for any points in the domain of $f$, we conclude that $-\log \operatorname{St}$ is not convex if $\vec{\Psi}$ is a linear transformation.
\end{proof}
The proof highlighted that the predictive distribution corresponds to a traditional multinomial logistic regression model when Dirichlet distributions are considered as model class, linear transformations as feature learning layers, and the negative log-likelihood as $L_1$.
It is well known that risk minimization in multinomial logistic regression is not strictly convex\footnote{http://deeplearning.stanford.edu/tutorial/supervised/SoftmaxRegression/}.
Concretely, subtracting a fixed vector $\vec{w}_0$ from each $\vec{w}_k$ does not affect \eqref{eq:logreg}:
\begin{equation*}
    \label{eq:logreg2}
    p(y=k \,|\, \vec{m}) = \frac{\mathrm{e}^{(\vec{w}_k-\vec{w}_0)^\top \vec{x}}}{\sum_{j=1}^K \mathrm{e}^{(\vec{w}_j-\vec{w}_0)^\top \vec{x}}} = \frac{\mathrm{e}^{\vec{w}_k^\top \vec{x}} \mathrm{e}^{-\vec{w}_0^\top \vec{x}}}{\sum_{j=1}^K \mathrm{e}^{\vec{w}_j^\top \vec{x}} \mathrm{e}^{-\vec{w}_0^\top \vec{x}}} = \frac{\mathrm{e}^{\vec{w}_k^\top \vec{x}}}{\sum_{j=1}^K \mathrm{e}^{\vec{w}_j^\top \vec{x}}} \,.
\end{equation*}
In other words, subtracting $\vec{w}_0$ from every $\vec{w}_k$ does not affect our hypothesis space at all.
The hypothesis space is overparameterized, in the sense that there are multiple parameter settings that give rise to exactly the same hypothesis function, which maps $\vec{x}$ to the predictions.
If the risk is minimized for some values of $\{\vec{w}_1,\ldots,\vec{w}_k\}$, then it is also minimized for $\{\vec{w}_1 - \vec{w}_0,\ldots,\vec{w}_k-\vec{w}_0\}$.
Thus, the risk minimizer is not unique. 
However, \eqref{eq:exprisk1} is still convex, and thus gradient descent will not run into local optima problems.
But the Hessian is non-invertible, which causes a straightforward implementation of Newton’s method to run into numerical problems. 

In multinomial logistic regression, this overparameterization is avoided by eliminating one of the vectors $\vec{w}_k$ (typically $\vec{w}_K$, but any $k$ could be chosen).
One then optimizes over the remaining $K-1$ parameter vectors.
However, for the setup considered in this paper, such a simplification is not possible, because it would fix epistemic uncertainty to a predefined level. 

\section{Closed-form outer loss minimization for exponential family members}
\label{sec:closedformouter}
In this appendix we show how the empirical risk in outer loss minimization can be simplified. For exponential family members, the log-likelihood can be written as 
\begin{equation*}
    \log p(y \,|\, \vec{\theta}) = \log h(y) + \vec{\theta}^\top \vec{u}(y) - A(\vec{\theta}) \,,
\end{equation*}
where $\vec{\theta} \in \mathbb{R}^L$ is the natural parameter,  $h: \mathbb{R} \to \mathbb{R}$ is the carrier or base measure, $A: \mathbb{R}^L \to \mathbb{R}$ the log-normalizer and $u : \mathbb{R} \to \mathbb{R}^L$ 
the sufficient statistics.

Its expectation under the prior distribution then becomes
\begin{equation*}
    \mathbb{E}_{\vec{\theta} \sim p(\vec{\theta} \,|\, \vec{m})} \log p(y \,|\, \vec{\theta}) = \log h(y) + \mathbb{E}_{\vec{\theta} \sim p(\vec{\theta} \,|\, \vec{m})} \!\left[ \vec{\theta}^\top \vec{u}(y) \right] - \mathbb{E}_{\vec{\theta} \sim p(\vec{\theta} \,|\, \vec{m})} A(\vec{\theta}) \,,
\end{equation*}
with $\mathbb{E}_{\vec{\theta} \sim p(\vec{\theta} \,|\, \vec{m})} \vec{\theta} = \vec{r}$. 

Thus, when $\mathcal{H}_2$ contains Dirichlet distributions, the expected log-likelihood function becomes 
\begin{align}
    \label{eq:outerclas}
    \mathbb{E}_{\vec{\theta} \sim p(\vec{\theta} \,|\, \vec{m})}\log p(y \,|\, \vec{\theta})
    &= 0 + \mathbb{E}_{\vec{\theta} \sim \operatorname{Dir}(\vec{m}) } \left[ \sum_{i=1}^K 1_{ [y=i ]} \log(\theta_i) \right] - 0 \nonumber \\
    &= \sum_{i=1}^K \mathbb{I}_{ [y=i ]} \, \mathbb{E}_{\vec{\theta} \sim \operatorname{Dir}(\vec{m}) } \left[ \log(\theta_i) \right] \nonumber \\
    &= \sum_{i=1}^K \mathbb{I}_{ [y=i ]} \left( \psi(m_i) - \psi \!\left( \sum_{j=1}^K m_j \right) \right), \nonumber \\
    &= \psi(m_y) - \psi\!\left( \sum_{j=1}^K m_j \right),
\end{align}
with $\mathbb{I}$ the indicator function, which returns one when its argument is true and zero otherwise. We used for the last inequality that the logarithmic moments\footnote{https://en.wikipedia.org/wiki/Dirichlet\_distribution\#Entropy} of a Dirichlet distribution are $\psi(m_i) - \psi(\sum_{j=1}^K m_j)$.
In a similar vain, \citet{Charpentier2022NaturalPN} derived that the expected log-likelihood is given by 
\begin{equation}
    \label{eq:outerreg}
    \mathbb{E}_{\vec{\theta} \sim p(\vec{\theta} \,|\, \vec{m})}\log p(y \,|\, \vec{\theta}) = \frac{1}{2} \left( -\frac{\alpha}{\beta} (y-\gamma)^2 -\frac{1}{\nu} +\psi(\alpha) -\log \beta -\log 2 \pi \right) 
\end{equation}
when $\mathcal{H}_2$ contains inverse-normal Gamma distributions. 

For Gamma distributions \citet{Charpentier2022NaturalPN} find 
\begin{equation*}
    \mathbb{E}_{\vec{\theta} \sim p(\vec{\theta} \,|\, \vec{m})}\log p(y \,|\, \vec{\theta}) = (\psi(\alpha) - \log \beta)y -\frac{\alpha}{\beta} - \sum_{k=1}^y \log k \,.
\end{equation*}


\begin{figure}
    \centering
\begin{subfigure}{\textwidth}
\centering
    \includegraphics[trim = 2 8 2 0, width = .9\textwidth, height=.27\textheight]{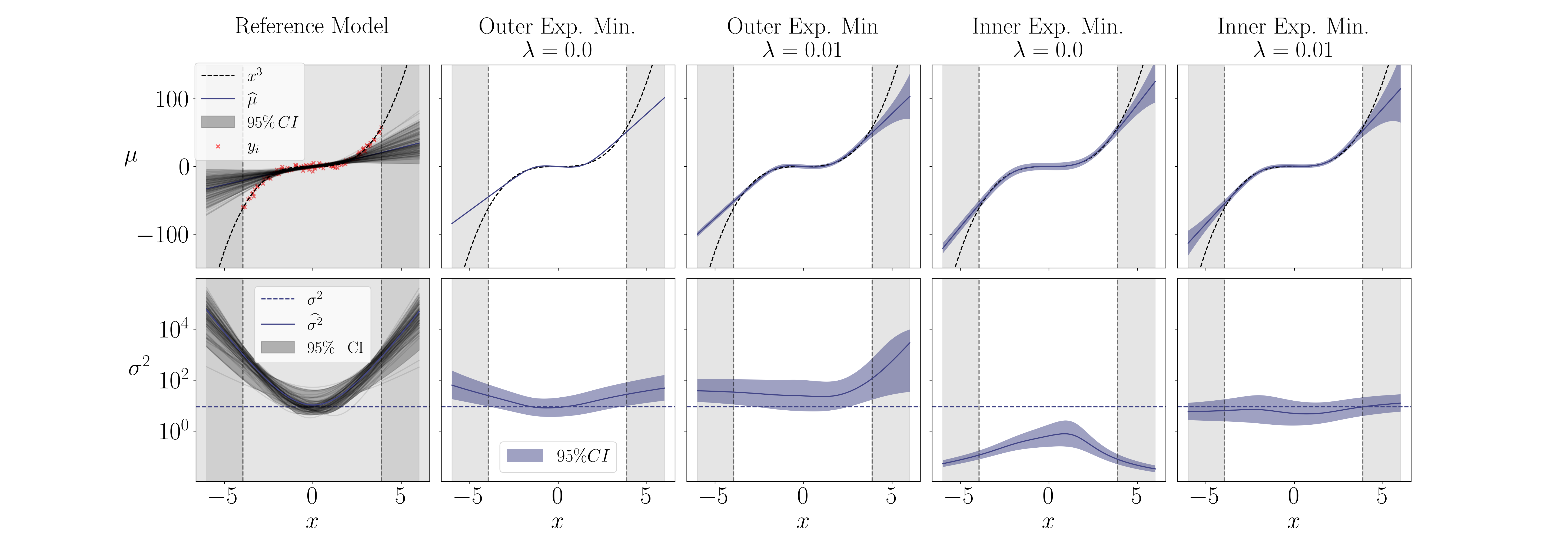}
\caption{$N=100$}
\end{subfigure}
\hfill
\begin{subfigure}{\textwidth}
\centering
    \includegraphics[trim = 2 8 2 0,width = .9\textwidth, height=.27\textheight]{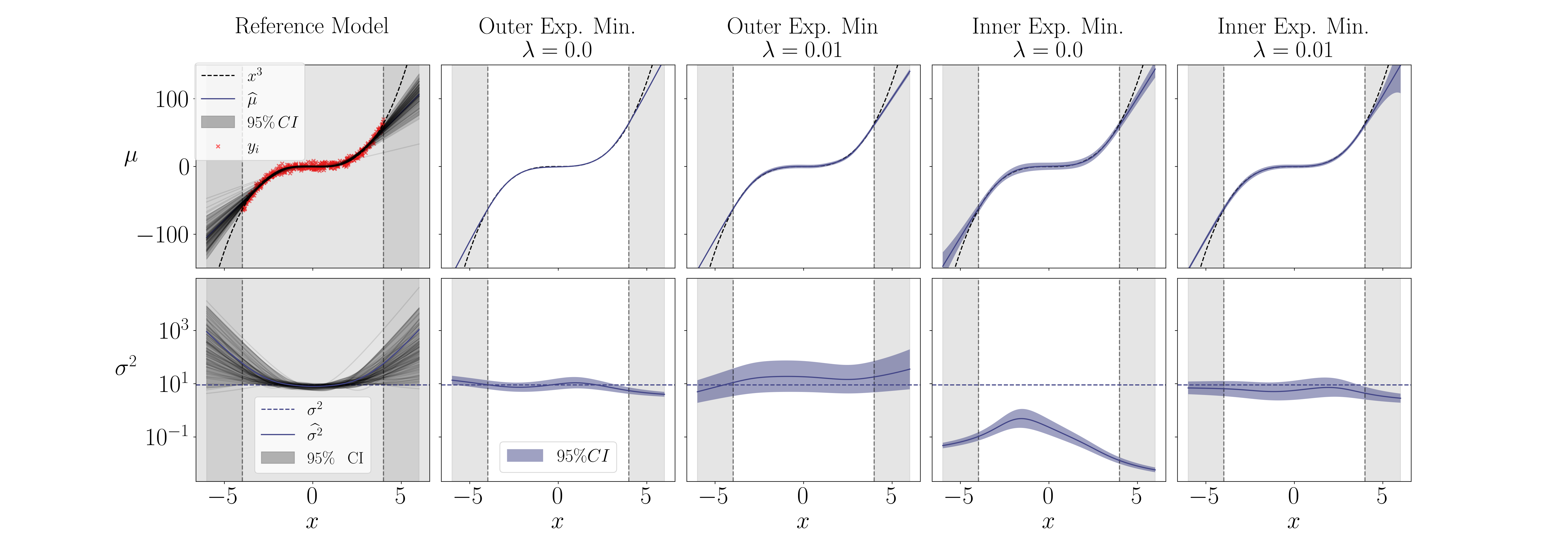}
\caption{$N=500$}
\end{subfigure}
\hfill
\begin{subfigure}{\textwidth}
\centering
\includegraphics[trim = 2 2 2 0,width=.9\textwidth, height=.27\textheight]{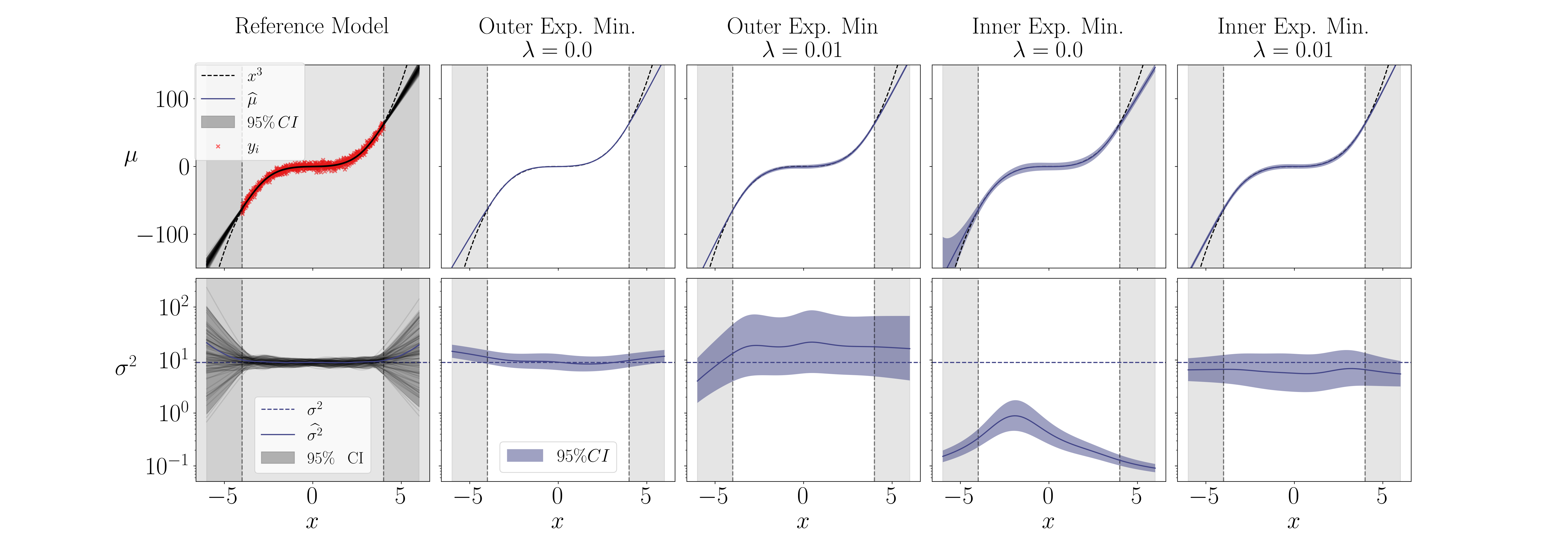}
\caption{$N=1000$}
\end{subfigure}
\caption{Regression experiments for training sample size $N \in \{100, 500, 1000\}$. The reference model learns the parameters $\theta=(\mu, \sigma)$ of the underlying normal distribution. The true underlying mean $\mu = x^3$ and variance $\sigma^2=9$, are visualised in separate rows together with the mean predictions (in blue) and obtained confidence intervals.  Confidence bounds for the predicted parameters by the reference model are obtained by resampling the training data $100$ times. Confidence bounds for $\mu$ and $\sigma^2$ learned by second-order risk minimization (visualized in purple) are obtained by the quantiles of the normal distribution $\mathcal{N}(\widehat{\mu}, \frac{\widehat{\beta}}{(\widehat{\alpha} -1)\widehat{\nu}})$, and of the Inverse-Gamma distribution $\Gamma^{-1}(\widehat{\alpha}, \widehat{\beta})$, respectively. The red dots denote one training dataset. See main text for experimental setup.}
\label{fig: regression conf bounds}
\end{figure}

\section{Proof of Theorem 3.2}
\label{sec:proofouter}
    The empirical risk is given by 
    \begin{equation*}
        \mathcal{R}(\vec{\Phi}) = \sum_{i=1}^N \mathbb{E}_{\vec{\theta} \sim p(\vec{\theta} \,|\, \vec{m}(\vec{x}_i; \vec{\Phi}))} L_1(y_i, p(y \,|\, \vec{\theta})) \,. 
    \end{equation*}
    When the first-order loss $L_1$ is convex, it follows from Jensen's inequality that 
    \begin{align*} 
        \mathcal{R}(\vec{\Phi}) \geq \sum_{i=1}^N L_1(y_i, \mathbb{E}_{\vec{\theta} \sim p(\vec{\theta} \,|\, \vec{m}(\vec{x}_i; \Phi))} p(y \,|\, \vec{\theta})) \,. 
    \end{align*}
    Let $\widetilde{\vec{\Phi}}$ be the minimizer of the right-hand side, then it holds for all $\vec{\Phi}$ that 
    \begin{align} 
        \label{eq:doubleineq}
                \mathcal{R}(\vec{\Phi}) 
                &\geq \sum_{i=1}^N L_1(y_i, \mathbb{E}_{\vec{\theta} \sim p(\vec{\theta} \,|\, \vec{m}(\vec{x}_i; \Phi))} p(y \,|\, \vec{\theta}))\\ 
                &\geq \sum_{i=1}^N L_1(y_i, \mathbb{E}_{\vec{\theta} \sim p(\vec{\theta} \,|\, \vec{m}(\vec{x}_i; \tilde{\vec{\Phi}}))} p(y \,|\, \vec{\theta})) \,. 
    \end{align}
    In what follows let
    \begin{equation*}
        \tilde{\vec{\theta}}(\vec{x}_i) = \mathbb{E}_{\vec{\theta} \sim p(\vec{\theta} \,|\, \vec{m}(\vec{x}_i; \widetilde \Phi))} \vec{\theta} \,.
    \end{equation*}
    For all $i \in \{1,\ldots,N\}$, let $\vec{\Phi}'$ denote the parameterization that corresponds to $ p(\vec{\theta} \,|\, \vec{m}(\vec{x}_i;\vec{\Phi}')) = \delta(\tilde{\vec\theta}(\vec{x_i}))$, where $\delta$ is the Dirac delta function.
    Such a parameterization always exists, because we assumed a hypothesis space of universal approximators.
    Furthermore, note that 
    \begin{equation*}
        \mathbb{E}_{\vec{\theta} \sim p(\vec{\theta} \,|\, \vec{m}(\vec{x}_i;\vec{\Phi}'))} \vec{\theta} \,=\, \tilde{\vec{\theta}}(\vec{x}_i) \,.
    \end{equation*}
    Then, 
    \begin{align*}
        \mathcal{R}(\vec{\Phi}') 
        &= \sum\limits_{i=1}^N \mathbb{E}_{\vec{\theta} \sim p(\vec{\theta} \,|\, \vec{m}(\vec{x}_i;\vec{\Phi}'))} L_1\left(y_i, p(y \,|\, \vec{\theta})\right) \\
        &= \sum\limits_{i=1}^N L_1\left(y_i, \tilde{\vec\theta}(\vec{x}_i) \right) \\
        &= \sum\limits_{i=1}^N L_1\left(y_i, \mathbb{E}_{\vec{\theta} \sim p(\vec{\theta} \,|\, \vec{m}(\vec{x}_i;\tilde{\vec{\Phi}}))} p(y \,|\, \vec{\theta})\right).
    \end{align*}
    Combining this last observation with inequality \eqref{eq:doubleineq} lets us conclude that for all $\vec{\Phi}$ it holds that $\mathcal{R}(\vec{\Phi}) \geq \mathcal{R}(\vec{\Phi}')$.
    Thus, $\vec{\Phi}'$ must be a solution of \eqref{eq:exprisk2}. 

\section{Proof of Theorem 3.3}
\label{sec:proofregular}
The proof is based on the fact that  \eqref{eq:exprisk1reg} and \eqref{eq:exprisk2reg} can be reformulated as constrained optimization problems:
\begin{eqnarray*}
    \hat{\vec{\Phi}} &=& \operatorname{arg}\min_{\vec{\Phi}} \sum_{i=1}^N L_1 \!\left( y_i, \mathbb{E}_{\vec{\theta} \sim p(\vec{\theta} \,|\, \vec{m}(\vec{x}_i; \vec{\Phi}))} \!\left[ p(y \,|\, \vec{\theta})\right] \right) \\
		&& \qquad \mbox{s.t. }  R(\vec{\Phi}) \leq t \\
    \hat{\vec{\Phi}} &=& \operatorname{arg} \min_{\vec{\Phi}} \sum_{i=1}^N \mathbb{E}_{\vec{\theta} \sim p(\vec{\theta} \,|\, \vec{m}(\vec{x}_i;\vec{\Phi}))} \!\left[ L_1(y_i, p(y \,|\, \vec{\theta})) \right] \\ 
		&&\qquad \mbox{s.t. }  R(\vec{\Phi}) \leq t
\end{eqnarray*}
with $t$ being a bijection of $\lambda$. 

To prove this equivalence\footnote{For the proof we found inspiration here: https://math.stackexchange.com/questions/416099/lasso-constraint-form-equivalent-to-penalty-form}, let us introduce the shorthand notation $\mathcal{R}(\vec{\Phi})$ for the empirical risk. In essence we intend to show equivalence between the \emph{penalty problem}
$$\operatorname{arg}\min_{\vec{\Phi}} \mathcal{R}(\vec{\Phi}) + \lambda  R(\vec{\Phi}) \\$$
and the \emph{constrained problem}
$$
\operatorname{arg}\min_{\vec{\Phi}} \mathcal{R}(\vec{\Phi}) \qquad  \qquad \mbox{s.t. }  R(\vec{\Phi}) \leq t \, .
$$
First, assume that $\hat{\vec{\Phi}}$ is a solution of the penalty problem, which implies for all $\vec{\Phi}$ that
\begin{equation}
\label{eq:lossineq}
\mathcal{R}(\vec{\Phi}) + \lambda  R(\vec{\Phi}) \geq \mathcal{R}(\hat{\vec{\Phi}}) + \lambda  R(\hat{\vec{\Phi}}) \,. 
\end{equation}
Hence for all $\vec{\Phi}$ such that $R(\vec{\Phi}) \leq t = R(\hat{\vec{\Phi}}) $ it holds that $\mathcal{R}(\vec{\Phi}) \geq \mathcal{R}(\hat{\vec{\Phi}})$. Thus, $\hat{\vec{\Phi}}$ is also a solution of the constrained problem when $t = R(\hat{\vec{\Phi}}) $. 

Conversely, assume that $\vec{\Phi}^{**}$ is a solution of the constrained problem with $t = R(\hat{\vec{\Phi}})$, which implies 
$$ R(\vec{\Phi}^{**}) \leq t = R(\hat{\vec{\Phi}}) \qquad \mbox{and} \qquad \mathcal{R}(\vec{\Phi}^{**}) \leq \mathcal{R}(\hat{\vec{\Phi}}) \,.$$
Since $\hat{\vec{\Phi}}$ is a solution of the penalty problem, inequality \eqref{eq:lossineq} holds. Combining those two inequalities leads to 
$$\mathcal{R}(\vec{\Phi}) + \lambda  R(\vec{\Phi}) \geq \mathcal{R}(\vec{\Phi}^{**}) + \lambda  R(\vec{\Phi}^{**}) \,, $$
for all $\vec{\Phi}$. This allows us to say that $\vec{\Phi}^{**}$ is also a solution of the penalty problem. 
As a result, we have proven that $\hat{\vec{\Phi}} = \vec{\Phi}^{**}$, so the two optimization problems are equivalent.  

The reformulation as a constrained optimization problem allows us to interpret the regularizer $R(\vec{\Phi})$ as an epistemic uncertainty budget that cannot be exceeded. As discussed in the main paper, when $t$ or $\lambda$ are chosen inadequately, the obtained solution will differ substantially from the reference distribution.

\section{Additional experiments}
\label{sec:regresults}
For the regression experiments we use the same setup as introduced in \cite{amini2020deep}, and critically analyzed by \citet{Meinert2022}. That is, we generate datasets $\{(x_i, x_i^3 + \epsilon)\}_{i=1}^N$ of different sample sizes, where the instances $x_i \in U([-4,4])$ are uniformly distributed and $\epsilon \sim \mathcal{N}(0, \sigma^2=9)$. Even though the setup is homoscedastic, we use negative log-likelihood loss for the first-order learner, while assuming an underlying normal distribution, making $\sigma$ as well as $\mu$ learnable parameters. This choice is made because the second-order risk minimization methods also assume heteroskedastic noise, so in this way the comparison is more fair.  

The second-order learner predicts the parameters $\boldsymbol{m}=(\gamma, \nu, \alpha, \beta)$ of a normal-inverse-Gamma distribution, trained with the loss functions obtained by outer and inner expectation minimization, respectively. Identically as in the experiment of Section \ref{sec:experiments}, we use two neural networks of the same architecture, with one fully-connected hidden layer of $32$ neurons. The first-order learner needs $2$ output neurons to predict $\boldsymbol{\theta}=(\mu, \sigma)$, while the second-order learners uses $4$ neurons for the parameters of the NIG distribution. The models are again trained for $5000$ epochs using the Adam optimizer and a learning rate of $0.0001$. Negative entropy of the NIG distribution is used as a regularization term in the loss function, where we set the weighting factor to $\lambda=0.01$.

Figure \ref{fig: regression conf bounds} displays the confidence bounds and mean prediction of the estimated reference distribution and the fitted second-order models.  The $95\%$ confidence bound for  $\widehat{\mu}$ and $\widehat{\sigma^2}$ are empirically computed for the reference model, and analytically for the second-order models, using the fact that $\mu \sim \mathcal{N}(\mu, \frac{\beta}{{(\alpha -1)\nu}})$ and $\sigma^2 \sim \Gamma^{-1}(\alpha, \beta)$ for the normal-inverse-Gamma distribution with parameters $\boldsymbol{m}=(\gamma, \nu, \alpha, \beta)$. We see that also in this setup, the mean prediction of the second-order learner remarkably well approximates the one of the reference model, however, the epistemic uncertainty is not well captured. For the outer expectation minimization, the convergence of the distribution of $\widehat{\mu}$ towards a Dirac delta distribution is clearly visible, while regularization does not improve the learning of the underlying epistemic uncertainty. The learned distribution of $\widehat{\sigma^2}$ is more difficult to interpret. Convergence to the Dirac delta distribution has not (yet) taken place. However, one has to keep in mind that all parameters are learned simultaneously by the neural network, and perhaps more is to be gained by focusing on $\mu$. Similarly as for the classification case, Figure \ref{fig: regression conf bounds} should be cautiously interpreted for the inner loss minimization method, because the figure only visualizes one run. 

Figure \ref{fig: parameter analysis nig} shows the learned parameters as a function of the number of epochs (up to 10000) for 40 runs with random initialization of the neural network.  For the unregularized case we see that the network learns to increase $\nu$ over time, thereby decreasing the variance of $\widehat{\mu}$.
For the inner expectation minimization,  we obtain similar results as in \cite{amini2020deep}. Note here that for comparability to the first experiment we use negative entropy instead of the proposed evidence based regularizor. For the unregularized case, the distribution of $\widehat{\sigma^2}$, whose expectation is an indicator of the aleatoric uncertainty, is more peaked in the center of region where the model sees data, which is in accordance to the results of \cite{Meinert2022}. Again the issue of non-identifiability of the learned parameters arises, and we obtained a high variability in the results for different runs.
\begin{figure}[h]
\includegraphics[width=\textwidth]{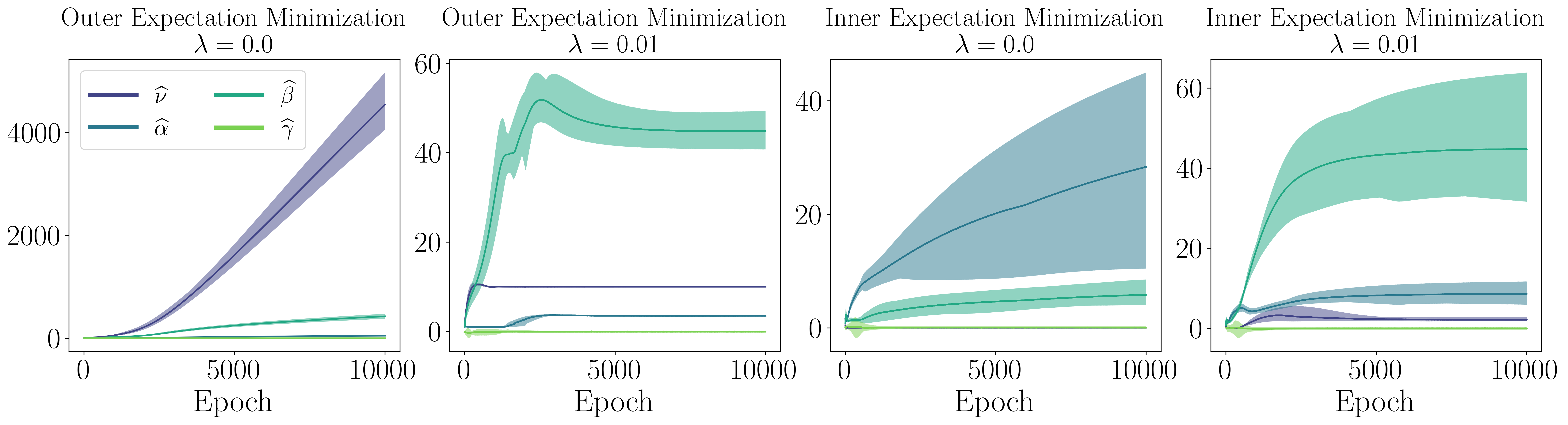}
\caption{Behavior of the average value of the predicted parameters $\widehat{\boldsymbol{m}}=(\widehat{\gamma}, \widehat{\nu}, \widehat{\alpha}, \widehat{\beta})$, obtained by the second-order learners as a function of the number of training epochs. To obtain the mean and the confidence bounds, $40$ models were trained on the same $N=1000$ training instances, differing only in their random weight initializations. The shaded areas are the $95\%$ empirical confidence bounds of the predictions, averaged over the instance space. In addition, the average over different runs is shown.}
\label{fig: parameter analysis nig}
\end{figure}

\end{document}